\relax
\documentclass[letterpaper]{article} 
\usepackage{aaai22}  
\usepackage{times}  
\usepackage{helvet}  
\usepackage{courier}  
\usepackage[hyphens]{url}  
\usepackage{graphicx} 
\usepackage{natbib}  
\usepackage{caption} 
\DeclareCaptionStyle{ruled}{labelfont=normalfont,labelsep=colon,strut=off} 
\usepackage{algorithm}
\usepackage{algorithmic}

\usepackage{newfloat}
\usepackage{listings}
\floatstyle{ruled}
\newfloat{listing}{tb}{lst}{}
\floatname{listing}{Listing}
\usepackage{amsmath}
\usepackage{amsthm}
\usepackage{multicol}
\usepackage{multirow}
\usepackage{xcolor,colortbl}

\newtheorem{proposition}{Proposition}
\newtheorem{example}{Example}

\usepackage{tikz}
\usetikzlibrary{shapes}
\usetikzlibrary{positioning}
\tikzset{
    roundnode/.style = {
        draw, 
        circle,
        minimum width=3mm,
        fill=white
    },
    squarednode/.style = {
        rectangle, 
        draw, 
        minimum width=5mm,
        fill=black!10 
    },
    nonode/.style = {
        circle,
        minimum width=5mm,
        fill=white
    },
    indicatornode/.style = {
        draw, 
        rectangle,
        minimum width = 10mm,
        fill=white
    },
    beadnodeone/.style = {
        draw,
        ellipse,
        minimum width=6mm,
        fill=white
    },
    beadnodetwo/.style = {
        draw,
        ellipse,
        minimum width=4mm,
        fill=white
    },
    beadnodethree/.style = {
        draw, 
        ellipse split,
        minimum width=6mm,
        fill=white
    }
}

\usepackage{afterpage}

\title{Optimizing Binary Decision Diagrams with MaxSAT for classification\footnote{This is the preprint version of the paper \cite{aaai-22}}}

\author{Hao Hu, Marie-José Huguet, Mohamed Siala
}
\affiliations{
LAAS-CNRS, Université de Toulouse, CNRS, INSA, Toulouse, France\\
\{hhu, huguet, siala\}@laas.fr
}

\begin{document}

\def\corr#1{\textcolor{red}{#1}} 
\def\corrhao#1{\textcolor{orange}{#1}}

\def\med#1{\textcolor{blue}{#1}}
\def\mjo#1{\textcolor{cyan}{#1}} 

\def\Sat{{SAT}}
\def\Cp{\texttt{CP}}
\def\Mip{\texttt{MIP}}
\def\MaxSat{{MaxSAT}}
\def\DT{\texttt{DT}}

\def\cart{\texttt{CART}}
\def\cff{\texttt{C4.5}}
\def\idt{\texttt{ID3}}
\def\dleight{\texttt{DL8.5}}
\def\Bdd{\texttt{BDD}}
\def\OODG{\texttt{OODG}}
\def\ODT{\texttt{ODT}}

\def\featset{\mathcal{F}}
\def\nbfeat{K}
\def\afeatind#1{f_{#1}}
\def\trainset{\mathcal{E}}
\def\nbtraindata{M}
\def\atraindataind#1{e_{#1}}
\def\atrainfeatind#1{x_{#1}}
\def\atrainfeatliteralind#1{\mathcal{L}_{#1}}
\def\aclassind#1{cl_{#1}}
\def\aclass{c}

\def\learnfunc{\gamma}
\def\actfunc{\phi}

\def\BddGraph{\mathcal{G}}
\def\nodeset{\mathcal{V}}
\def\boolfunc{g}
\def\node{v}
\def\root{\textit{root}}
\def\varset{\mathcal{X}}
\def\avarind#1{x_{#1}}
\def\nbvar{n}
\def\index#1{\textit{index($#1$)}}
\def\Left#1{\textit{left($#1$)}}
\def\Right#1{\textit{right($#1$)}}
\def\Value#1{\textit{value($#1$)}}

\def\Truthtable{\beta}
\def\Bead{\textit{bead}}
\def\Beads{\textit{beads}}
\def\Subtable{\textit{subtable}}
\def\Subtables{\textit{subtables}}
\def\DuplicateSubTT{\alpha}

\def\satvarset{\mathcal{X}}
\def\satvar{var}
\def\satdom{Dom}
\def\satvalue{val}

\def\depth{H}
\def\treesize{N}

\def\featplacevar#1#2{a_{#1}^{#2}}
\def\adecisionind#1{c_{#1}}
\def\signfeatexample#1#2{d_{#1}^{#2}}

\def\dirfunc{\textit{rel}}
\def\signfunc{\textit{sign}}

\def\boolfuncensemble#1#2{\mathcal{S}^{#1}_{#2}}

\newcommand{\comment}[1]{}

\maketitle

\begin{abstract}


The growing interest in {explainable artificial intelligence (XAI)} for critical decision making motivates the need for
interpretable machine learning (ML) models.
In fact, due to their structure (especially with small sizes), these models are inherently understandable by humans.
Recently, several exact methods for computing such models are proposed to overcome weaknesses of traditional heuristic methods 
by providing more compact models or better prediction quality.

Despite their compressed representation of Boolean functions, Binary decision diagrams (\Bdd{}s) did not gain enough interest
as other interpretable ML models.
In this paper, we first propose
SAT-based models for learning optimal \Bdd{}s (in terms of the number of features) that classify all input examples.
Then, we lift the encoding
to a MaxSAT model to learn optimal \Bdd{}s in limited depths,
that maximize the number of examples correctly classified. 
Finally, 
we tackle the fragmentation problem by
introducing a method to merge compatible subtrees for the \Bdd{}s found via the MaxSAT model. 
Our empirical study shows clear benefits of the proposed approach in terms of prediction quality and intrepretability (i.e., lighter size) compared to the state-of-the-art approaches. 



\end{abstract}

\section{Introduction}

\noindent
Due to
the increasing concerns in understanding the reasoning behind AI decisions for critical applications,
interpretable Machine Learning (ML) models gained a lot of attention.
Examples of such ML
applications include job recruitment, bank credit applications, and justice~\cite{EUlaw}. 
%
%
%
Most of traditional
approaches for building interpretable  models are greedy,
for example, decision trees~\cite{DBLP:books/wa/BreimanFOS84, 10.1023/A:1022643204877/Quinlan1986, quinlan1993c45},
rule lists~\cite{DBLP:conf/icml/Cohen95, 10.1007/BFb0017011}, and
decision sets~\cite{DBLP:conf/kdd/LakkarajuBL16}.
Compared to traditional approaches, 
exact methods offer guarantee of optimality, such as model size and accuracy.
In this context, 
combinatorial optimisation methods, such as Constraint Programming~\cite{Bonfietti2015, DBLP:journals/constraints/VerhaegheNPQS20}, Mixed Integer Programming~\cite{Angelino2018learning, aaai/Zhang19, DBLP:conf/aaai/AglinNS20}, or Boolean Satisfiablility (\Sat{})~\cite{bessiere-cp09, DBLP:conf/ijcai/NarodytskaIPM18, DBLP:conf/aaai/Avellaneda20, DBLP:conf/ijcai/Hu0HH20, DBLP:conf/sat/JanotaM20, DBLP:conf/cp/YuISB20} have been successfully used to learn interpretable models.
These declarative approaches are particularly interesting since they offer certain flexibility to handle additional requirements when learning a model. 

Decision Trees are widely used as a standard interpretable model.
However, 
they suffer from two major flaws: \textit{replication} and \textit{fragmentation}~\cite{89-ijcai-duplication,90-ml, Rokach-14-dm-dt-book}. The \textit{replication} problem appears when two identical subtrees are in the decision tree. 
The \textit{fragmentation} problem appears when only few samples are associated to leave nodes.
By providing compact representations for Boolean functions, 
Binary Decision Diagrams (\Bdd{}s)~\cite{78-bdd, 82-bdd, 86-bdd, KnuthTAOCP4} are widely studied for hardware design, model checking, and knowledge representation.
In the context of ML, 
\Bdd{} could be viewed as an intrepretable model for binary classification. In addition, they were extended for multi-classification, known as \textit{decision graphs} and
heuristic methods were proposed in~\cite{93-decisiongraph, DBLP:conf/ecml/Kohavi94, DBLP:conf/ijcai/KohaviL95, DBLP:conf/diagrams/MuesBFV04}. 
Moreover,~\cite{DBLP:conf/ictai/IgnatovI17} proposed \textit{decision stream}, a similar topology to \Bdd{} based on merging \textit{similar} subtrees in each split made in decision trees to improve the generalization.
 \cite{93-decisiongraph, DBLP:conf/ecml/Kohavi94} showed that \textit{decision graphs} could avoid the \textit{replication} problem and \textit{fragmentation} problem {of decision trees} effectively 
For binary classification, \Bdd{}s also could avoid these problems.
This fact indicates that generally in the practice of ML, a \Bdd{} have a smaller size than the corresponding decision tree.




In this paper, we introduce a \Sat{}-based model 
for learning optimal \Bdd{}s with the smallest number of features classifying all examples correctly,
and a lifted \MaxSat{}-based model to learn optimal \Bdd{}s minimizing the classification error.
We assume that all \Bdd{}s are \textit{ordered} and \textit{reduced}\footnote{The two notions are defined in the background}, the limitation on the depth for a \Bdd{}, corresponds to the number of features to be selected by our model.
To the best of our knowledge,~\cite{DBLP:conf/date/CabodiCI0PP21} is the only exact method of learning optimal \Bdd{}s in the context of ML.
The authors proposed a \Sat{} model to learn optimal 
\Bdd{}s with the smallest sizes that correctly classify all examples.
In their approach, the depth of the \Bdd{} is not restrained.
In fact, it is possible that the constructed \Bdd{} is small in size (number of nodes) and high in depth.
As the \Bdd{} is \textit{ordered}, this approach could not limit the number of features used,
making it not quite comparable with our proposition.
Another related work is in~\cite{DBLP:conf/ijcai/Hu0HH20} where the authors consider a \MaxSat{} model to learn optimal decision trees minimizing the classification error within a limited depth.
The usage of the same solving methodology with the same objective function and 
the depth limit,
makes these two \MaxSat\ models comparable.
Finally, in order to increase the scalability of our approach, we propose a heuristic extension based on a simple pre-processing step. 


The rest of the paper is organized as follows.
First, we give the related technical background in Section~\ref{sec:background}.
Then, in Section~\ref{sec:Sat-models}, we  show the proposed \Sat{} and \MaxSat{}
models for learning \Bdd s in binary classification. 
Finally in Section \ref{sec:exp}, we present our large experimental studies to show the competitive prediction quality of the proposed approach. 


\section{Technical Background}

\label{sec:background}



\subsection{Classification} 


Consider a dataset $\trainset =\{\atraindataind{q}, \dots, \atraindataind{\nbtraindata}\}$
with $\nbtraindata$ examples.
Each example $\atraindataind{q} \in \trainset$ is characterized
by a list of binary features 
$\atrainfeatliteralind{q} = [\afeatind{1}, \dots, \afeatind{\nbfeat} ]$
and a binary target $\aclassind{q}$, representing the class of the example ($\aclassind{q} \in \{0,1 \}$).
The data set is partitioned into $\trainset^+$ and $\trainset^-$, where $\trainset^+$ (respectfully $\trainset^-$) is the set of positive (respectifully negative) examples. That is, $\aclassind{q} = 1$ iff $\atraindataind{q} \in \trainset^+$ and $\aclassind{q}=0$ iff $\atraindataind{q} \in \trainset^-$.
%
%
We assume that, $\forall 1 \leq q, q' \leq \nbtraindata$,  
$\atrainfeatliteralind{q}=\atrainfeatliteralind{q'}$ implies $\aclassind{q} = \aclassind{q'}$.

Let $\actfunc$ be the function defined by 
$\actfunc(\atrainfeatliteralind{q})= \aclassind{q}$, $\forall q \in [1,\nbtraindata]$. 
The classification problem is to compute
a function $\learnfunc$ (called a \emph{classifier})
that matches as accurately as possible the function $\actfunc$ on examples $\atraindataind{q}$
of the training data and generalizes well on unseen test data. 

\subsection{Binary Decision Diagrams}
Binary Decision Diagrams (\Bdd{}s) are used to provide compact representation of Boolean functions. 
Let $[\avarind{1}, \dots, \avarind{\nbvar} ]$ 
be a sequence of 
of $\nbvar$ Boolean variables.
A \Bdd{} is a rooted, directed, acyclic graph $\BddGraph{}$.
The vertex set $\nodeset{}$ of $\BddGraph{}$ contains two types of vertices. 
A \emph{terminal} vertex $\node{}$ is associated to a binary value: \Value{\node{}} $\in \{0, 1\}$.
A \emph{nonterminal} vertex $\node{}$,  is associated to a Boolean variable $\avarind{i}$ and has two children \Left{\node{}}, \Right{\node{}} $\in \nodeset{}$. 
In this case, $\index{\node{}} =i \in \{1, \dots, \nbvar\}$ is the index of the Boolean variable associated to $\node{}$.

We assume that all \Bdd{}s are \emph{ordered} and \emph{reduced}.
These two restrictions are widely considered in the literature as they guarantee 
a \textit{unique} \Bdd{} for a given Boolean function.
The restriction \emph{ordered} indicates that for any \emph{nonterminal} vertex $\node{}$, 
\index{\node{}} $<$ \index{\Left{\node{}}} and \index{\node{}} $<$ \index{\Right{\node{}}}. 
The restriction \emph{reduced} indicates that the graph contains no \emph{nonterminal} vertex $\node{}$ with \Left{\node{}} $=$ \Right{\node{}}, 
nor does it contain distinct \emph{nonterminal} vertices $\node{}$ and $\node{}'$ 
{having isomorphic rooted sub-graphs.}
Therefore, given an \emph{ordered} \emph{reduced} {\Bdd{}} $\BddGraph{}$ with \root{} $\node{}$, the {associated} Boolean function 
can be recursively obtained with the Shannon expansion process~\cite{ShannonSymbolic}. 






Let $\boolfunc{}$ be a Boolean function defined over a sequence 
$\varset = [\avarind{1}, \dots, \avarind{\nbvar}]$ of $\nbvar$ Boolean variables.
The function $\boolfunc{}$ 
can be represented by a 
\emph{truth table} 
that lists 
the $2^\nbvar$ values of all assignments of the $\nbvar$ variables.
The value of the truth table is therefore associated to a string of $2^\nbvar$ binary values.
{A truth table $\Truthtable{}$ of length $2^\nbvar$ is said to be of order $\nbvar$.}
A truth table $\Truthtable{}$ of order $\nbvar > 0$ 
has the form $\Truthtable{}_0\Truthtable{}_1$, where $\Truthtable{}_0$ and $\Truthtable{}_1$ are truth tables of order $\nbvar -1$, and
$\Truthtable{}_0$ and $\Truthtable{}_1$ are called \Subtables{} of $\Truthtable{}$.
The \Subtables{} of \Subtables{} are also considered to be \Subtables{}, and a table is considered as a \Subtable{} of itself.
A \Bead{} of order $\nbvar$ is a truth table $T$ of order $\nbvar$ that {does not} have the form $\DuplicateSubTT{}\DuplicateSubTT{}$  
{where $\DuplicateSubTT{}$}
is a subtable of $T$.
The \Beads{} {of} $\boolfunc{}$ are the \Subtables{} of its truth table that happen to be \Beads{}.
Proposition \ref{prop:bead_bdd_node} from \cite{KnuthTAOCP4}  relates truth table and binary decision diagram 
for the same Boolean function.

\begin{proposition}
\label{prop:bead_bdd_node}
    All vertices in $\nodeset{}$ of a binary decision diagram $\BddGraph{}$,
    are in one-to-one correspondence with the \Beads{} of the Boolean function $\boolfunc{}$ it represents.
\end{proposition}

Based on Proposition \ref{prop:bead_bdd_node}, 
we can produce the ordered and reduced  binary decision diagram of a Boolean function by finding its \Beads{} and combine its \Beads{} with its sequence of variables.
An algorithm for producing the corresponding \Bdd{} is provided in the Appendix.
\begin{example}
\label{exp:bdd_example}
Consider the Boolean function from~\cite{KnuthTAOCP4}: $\boolfunc{}_1(\avarind{1}, \avarind{2}, \avarind{3}) = (\avarind{1} \lor \avarind{2}) \land (\avarind{2} \lor \avarind{3}) \land (\avarind{1} \lor \avarind{3})$. 
The binary string associated to its
truth table $\Truthtable{}$ is $00010111$.
The \Beads{} of $\Truthtable{}$
are $\{00010111, 0001, 0111, 01, 0, 1\}$.

From Proposition \ref{prop:bead_bdd_node}, we can draw the \Bdd{} with the beads found, shown as the left part of Figure~\ref{fig:Ex1_bdd}.
The dashed (solid) line of each vertex indicates the left (right) child. 
Then, we can replace the beads by vertices associated with the sequence of Boolean variables.
The final binary decision diagram for $\boolfunc{}_1$ is shown as the right part of Figure~\ref{fig:Ex1_bdd}.
\end{example}
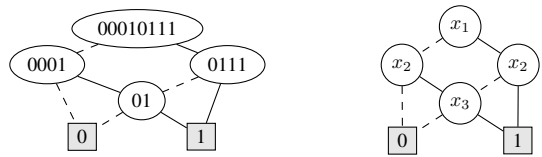
\begin{figure}[htb]
    \centering
    \begin{minipage}{0.54\linewidth}
        \centering
                \begin{tikzpicture}[scale=0.75, every node/.style={transform shape}]
        \node[beadnodeone]    (node1)     {00010111};
        \node[beadnodetwo]    (node2)     [below left = 1.5mm and 3mm of node1]     {0001};
        \node[beadnodetwo]    (node3)     [below right=1.5mm and 3mm of node1]     {0111};
            \node[beadnodetwo]    (node4)     [below left = 1.5mm and 7.5mm of node3]     {01};
            \node[squarednode]  (node5)     [below left = 1.5mm and 5mm of node4]     {0};
            \node[squarednode]  (node6)     [below right= 1.5mm and 5mm of node4]     {1};
            \draw           (node1) -- (node3);
            \draw [dashed]  (node1) -- (node2);
            \draw           (node2) -- (node4);
            \draw [dashed]  (node2) -- (node5);
            \draw           (node3) -- (node6);
            \draw [dashed]  (node3) -- (node4);
            \draw           (node4) -- (node6);
            \draw [dashed]  (node4) -- (node5);
        \end{tikzpicture}
    \end{minipage} \hfill
    \begin{minipage}{0.44\linewidth}
        \centering
        \begin{tikzpicture}[scale=0.75, every node/.style={transform shape}]
            \node[roundnode]    (node1)     {$x_{1}$};
            \node[roundnode]    (node2)     [below left = 1.5mm and 5mm of node1]     {$x_{2}$};
            \node[roundnode]    (node3)     [below right= 1.5mm and 5mm of node1]     {$x_{2}$};
            \node[roundnode]    (node4)     [below left = 1.5mm and 5mm of node3]     {$x_{3}$};
            \node[squarednode]  (node5)     [below left = 1.5mm and 5mm of node4]     {0};
            \node[squarednode]  (node6)     [below right= 1.5mm and 5mm of node4]     {1};
            \draw           (node1) -- (node3);
            \draw [dashed]  (node1) -- (node2);
            \draw           (node2) -- (node4);
            \draw [dashed]  (node2) -- (node5);
            \draw           (node3) -- (node6);
            \draw [dashed]  (node3) -- (node4);
            \draw           (node4) -- (node6);
            \draw [dashed]  (node4) -- (node5);
        \end{tikzpicture}
    \end{minipage}
    \caption{\footnotesize{The Binary decision diagram for $\boolfunc{}_1(\avarind{1}, \avarind{2}, \avarind{3})$}}
    \label{fig:Ex1_bdd}
\end{figure}


\subsection{Oblivious Read-Once Decision Graphs}
Oblivious Read-Once Decision Graphs (\OODG{}s) 
are proposed in~\cite{DBLP:conf/ecml/Kohavi94} to overcome some limitations of
decision trees for multi-classification, like \textit{replication} and \textit{fragmentation} problem.
We refer the readers to \cite{DBLP:conf/ijcai/KohaviL95, DBLP:conf/ecml/Kohavi94} for details on \OODG{}s.
An \OODG{} is a rooted, directed, acyclic graph, which contains terminal \textit{category nodes} labelled with classes to make decisions, and non-terminal \textit{branching nodes} labelled with features to make splits.
The property ``\textit{read-once}'' indicates that each feature occurs at most once along any path from the root to a category node. 
The property ``\textit{levelled}'' indicates that the nodes are partitioned into a sequence of pairwise disjoint sets, representing the levels, such that outgoing edges from each level terminate 
at the next level.
The property
``\textit{oblivious}'' extends the idea of ``\textit{levelled}'' by guaranteeing that all nodes at a given level are labelled by the same feature.

For the classification process, top-down and bottom-up heuristic methods for building \OODG{}s are proposed in~\cite{DBLP:conf/ijcai/KohaviL95, DBLP:conf/ecml/Kohavi94}. 
Here, we introduce briefly the top-down heuristic method,
which is similar to the heuristic methods \cff{} and \cart{} for computing decision trees.
The top-down heuristic induction for \OODG{} with given depth
contains three critical phases: 
(1) selecting a sequence of features with the help of \textit{mutual information} (the difference of \textit{conditional entropy} \cite{elementInformationtheory});
(2) growing an oblivious decision tree (\ODT{}) by splitting the dataset with features in the sequence selected;
and (3) merging \textit{isomorphic} and \textit{compatible} subtrees from top to down 
to build the \OODG{}.
When building the \ODT{}, the algorithm marks nodes that capture no example of the dataset as ``\textit{unknwon}''.
For the merging phase, two subtrees are \textit{compatible} if at least one root is labelled as ``\textit{unknown}'', or if the two root nodes are labelled with same feature and their corresponding children are the roots of compatible subtrees.
The \ODT{} grown could make classifications directly by assigning ``\textit{unknown}'' nodes with the majority class of their parents.

\begin{figure}[htb]
    \centering
    \begin{minipage}{0.21\linewidth}
        \centering
        \begin{tikzpicture}[scale=0.72, every node/.style={transform shape}]
            \node[roundnode]    (node1)     {$\afeatind{1}$};
            \node[squarednode]  (node2)     [below left = 4mm and 1mm of node1]     {u};
            \node[roundnode]    (node3)     [below right= 4mm and 0.5mm of node1]     {$\afeatind{2}$};
            \node[squarednode]    (node4)   [below left = 6mm and 10mm of node3]     {1};
            \node[squarednode]  (node5)     [below left = 6.5mm and 2mm of node3] {u};
            \node[squarednode]  (node6)     [below = 5mm of node3]     {1};
            \draw  (node1) -- (node3) node [midway, above, sloped] (TextNode) {yes} ;
            \draw  (node1) -- (node2) node [midway, above, sloped] (TextNode) {no} ;
            \draw  (node3) -- (node4) node [midway, above, sloped] (TextNode) {$v_1$} ;
            \draw  (node3) -- (node5) node [midway, above, sloped] (TextNode) {$v_2$} ;
            \draw  (node3) -- (node6) node [midway, above, rotate=-90] (TextNode) {$v_3$} ;            
\end{tikzpicture}
    \end{minipage} 
    \begin{minipage}{0.375\linewidth}
        \centering
        \begin{tikzpicture}[scale=0.72, every node/.style={transform shape}]
            \node[roundnode]    (node1)     {$\afeatind{1}$};
            \node[roundnode]  (node2)     [below left = 4mm and 5.7mm of node1]     {$\afeatind{2}$};
            \node[roundnode]    (node3)     [below right= 4mm and 5.7mm of node1]     {$\afeatind{2}$};
            \node[squarednode]    (node4)   [below left = 6mm and 2mm of node2]     {1};
            \node[squarednode]  (node5)     [below = 5mm of node2] {1};
            \node[squarednode]  (node6)     [below right = 6mm and 2mm of node2]     {1};
            \node[squarednode]  (node7)     [below left= 6mm and 2mm of node3]     {1};
            \node[squarednode]  (node8)     [below = 5mm of node3]     {0};
            \node[squarednode]  (node9)     [below right= 7mm and 2mm of node3]     {u};
            \draw  (node1) -- (node3) node [midway, above, sloped] (TextNode) {yes} ;
            \draw  (node1) -- (node2) node [midway, above, sloped] (TextNode) {no} ;
            \draw  (node2) -- (node4) node [midway, above, sloped] (TextNode) {$v_1$} ;
            \draw  (node2) -- (node5) node [midway, above, rotate=-90] (TextNode) {$v_2$} ;
            \draw  (node2) -- (node6) node [midway, above, sloped] (TextNode) {$v_3$} ;
            \draw  (node3) -- (node7) node [midway, above, sloped] (TextNode) {$v_1$} ;
            \draw  (node3) -- (node8) node [midway, above, rotate=-90] (TextNode) {$v_2$} ;
            \draw  (node3) -- (node9) node [midway, above, sloped] (TextNode) {$v_3$} ;
\end{tikzpicture}
    \end{minipage} \vline\hfill
    \begin{minipage}{0.38\linewidth}
        \centering
        \begin{tikzpicture}[scale=0.72, every node/.style={transform shape}]
            \node[roundnode]    (node1)     {$\afeatind{1}$};
            \node[roundnode]  (node2)     [below left = 4mm and 5.7mm of node1]     {$\afeatind{2}$};
            \node[roundnode]    (node3)     [below right= 4mm and 5.7mm of node1]     {$\afeatind{2}$};
            \node[squarednode]    (node4)   [below left = 6mm and 2mm of node2]     {1};
            \node[squarednode]  (node5)     [below = 5mm of node2] {1};
            \node[squarednode]  (node6)     [below right = 6mm and 2mm of node2]     {1};
            \node[squarednode]  (node7)     [below left= 6mm and 2mm of node3]     {1};
            \node[squarednode]  (node8)     [below = 5mm of node3]     {0};
            \node[squarednode]  (node9)     [below right= 6.2mm and 2mm of node3]     {1};
            \draw  (node1) -- (node3) node [midway, above, sloped] (TextNode) {yes} ;
            \draw  (node1) -- (node2) node [midway, above, sloped] (TextNode) {no} ;
            \draw  (node2) -- (node4) node [midway, above, sloped] (TextNode) {$v_1$} ;
            \draw  (node2) -- (node5) node [midway, above, rotate=-90] (TextNode) {$v_2$} ;
            \draw  (node2) -- (node6) node [midway, above, sloped] (TextNode) {$v_3$} ;
            \draw  (node3) -- (node7) node [midway, above, sloped] (TextNode) {$v_1$} ;
            \draw  (node3) -- (node8) node [midway, above, rotate=-90] (TextNode) {$v_2$} ;
            \draw  (node3) -- (node9) node [midway, above, sloped] (TextNode) {$v_3$} ;
\end{tikzpicture}
    \end{minipage}
    \caption{\footnotesize{An example of two compatible subtrees (the left two) and the merged tree (the right one) from~\cite{DBLP:conf/ijcai/KohaviL95}}}
    \label{fig:Ex4_compatible_oodg}
\end{figure}
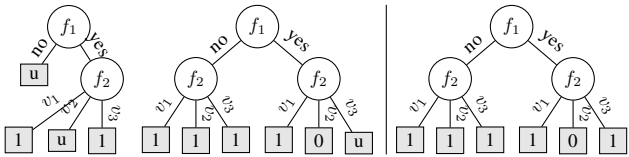


Figure \ref{fig:Ex4_compatible_oodg} shows an example of two compatible subtrees and the merged tree, where ``\textit{unknown}'' nodes are labelled as ``\textit{u}''.
Merging compatible subtrees changes the bias by assuming that a ``\textit{unknown}'' node is likely to behave the same as another child if they belong to compatible subtrees.

In binary classification for binary datasets, \OODG{}s could be considered \textit{equivalent} to \Bdd{}s, as the properties ``\textit{oblivious}'' and ``\textit{read-once}'' for \OODG{}s are same as property ``\textit{ordered}'' for \Bdd{}s.
In addition, the use of merging compatible subtrees could also be applied for \Bdd{}s.


\subsection{SAT \& MaxSAT}

We use standard terminology for Boolean Satisfiabily~\cite{DBLP:series/faia/2009-185}. A \emph{literal} is a Boolean variable or its negation, and a \emph{clause} is a disjunction of literals.
An assignment of variables satisfies a clause if one of its literals is true.
Given a set of Boolean variables 
and a set of clauses defined over these variables,
the \Sat{} problem can be defined as finding an assignment of the variables such that all the clauses are satisfied. 
Maximum Satisfiability (\MaxSat{}) is an optimization version of the \Sat{} problem, where 
the clauses are partitioned into \emph{hard} and \emph{soft} clauses.
Here we consider the Partial \MaxSat{} problem, that is to find an assignment of the Boolean variables that satisfies all the hard clauses and maximizes the number of satisfied soft clauses.

\section{(Max)SAT-based model for Binary Decision Diagrams}
\label{sec:Sat-models}
%
In this section, we present our approach for learning \Bdd {}s for binary classification using \Sat{} and \MaxSat{}.

\subsection{Problem Definition}
\label{sec:problemdefinition}

We firstly consider the following decision problem for classification with \Bdd{} in a given depth. 
\begin{itemize}
    \item $P_{{bdd}}(\trainset, \depth):$ \emph{Given a set of examples $\trainset$, is there a \Bdd\ of depth $\depth$ that classifies correctly all examples in $\trainset$?}
\end{itemize}
Notice that the algorithm for $P_{bdd}(\trainset, \depth)$ can be used to the alternative problem of optimizing a \Bdd{} that classifies all examples in the dataset correctly with a minimum depth. 
For that purpose, one can use a linear search that takes 
an initial depth $\depth_0$ as input 
and progressively increases or decreases this value depending on the result of solving $P_{bdd}(\trainset, \depth)$. 

Next, we consider another optimization problem for the classification with \Bdd{} 
in a limited depth.
\begin{itemize}
    \item $P_{{bdd}}^*(\trainset, \depth):$ \emph{Given a set of examples $\trainset$, find a \Bdd\ of depth $\depth$ that maximises the number of examples in $\trainset$ that are correctly classified.}
\end{itemize}

We propose an initial \Sat\ model for the decision problem $P_{bdd}(\trainset, \depth)$.
Then, we propose an improved version in tighter formula size.
Finally, we show how the improved \Sat{} model for $P_{bdd}(\trainset, \depth)$  can be used effectively to solve
the optimization problem $P_{{bdd}}^*(\trainset, \depth)$ with \MaxSat{}.



\subsection{SAT Model for $P_{bdd}(\trainset, \depth)$}

As shown in Section~\ref{sec:background}, a \Bdd{} of depth $\depth$ could be generated from the combination of 
a sequence of Boolean variables of size $\depth$: $[\avarind{1}, \dots, \avarind{\depth}]$, and a truth table of order $\depth$ associated to a Boolean function.
To solve the classification problem $P_{bdd}(\trainset, \depth)$, 
we then have to find a sequence of binary features of size $\depth$
that maps one-to-one the
sequence of Boolean variables, and a truth table
associated to a Boolean function that well-classified all examples.
We denote the sequence of binary features found as \textit{feature ordering}.
%
%
Therefore, the \Sat\ encoding consists of two parts:

\begin{itemize}
    \item \textbf{Part 1:} Constraints for selecting features of the dataset into the feature ordering of size $\depth$.
    \item \textbf{Part 2:} Constraints for generating a truth table that classifies all examples of $\trainset$ correctly with the selected feature ordering. 
\end{itemize}

To realize the \Sat{} encoding, we introduce two sets of Boolean variables as follow:
\begin{itemize}
    \item 
    $\featplacevar{r}{i}$: 
    the variable $\featplacevar{r}{i}$ is $1$ iff feature $\afeatind{r}$ is selected as $i$-th feature in the feature ordering, 
    where $i=1,\dots, \depth$, $r=1,\dots,\nbfeat$.
    \item 
    $\adecisionind{j}$: 
    the variable $\adecisionind{j}$ is 1 iff the $j$-th 
    value of the truth table is 1, where $j=1,\dots,2^{\depth}$.
\end{itemize}


The set of variables $\featplacevar{r}{i}$ guarantees the \textit{ordered} restriction.
Then, we introduce two constraints (\ref{con:feature_used_atmost_once}) and~(\ref{con:exact_one_feature})
for the feature ordering. 
Constraint~\ref{con:feature_used_atmost_once} ensures that any feature $\afeatind{r}$ can be selected at most once.
\begin{equation}
    \label{con:feature_used_atmost_once}
    \sum_{i=1}^{\depth}\featplacevar{r}{i} \leq 1, \quad r=1,\dots,\nbfeat
\end{equation}
Then, there is exacty one feature selected for each index of the feature ordering.
\begin{equation}
    \label{con:exact_one_feature}
    \sum_{r=1}^{\nbfeat}\featplacevar{r}{i} = 1, \quad i = 1,
    \dots, \depth
\end{equation}

We use the classical sequential counter encoding proposed in~\cite{05-sequential} to model constraints~(\ref{con:feature_used_atmost_once}) and~(\ref{con:exact_one_feature}) as a Boolean formula.

The truth table we are looking for is the binary string of the values of variables
$\adecisionind{1}\adecisionind{2}\dots\adecisionind{2^\depth}$.
To avoid the first feature selected makes useless split, we need to make sure that the truth table is a \Bead{}.
\begin{equation}
    \label{con:root_is_bead}
    \bigvee_{j=1}^{2^{\depth-1}}(\adecisionind{j} \oplus \adecisionind{j + 2^{\depth-1}})
\end{equation}

There is a relationship between the values of a truth table and the assignments of the given sequence of Boolean variables.
For example, the first value of a truth table corresponds to the assignment that 
$\avarind{1} = 0$ and $\avarind{2} =0$.
%
Therefore, we define the following function to obtain the value of the $i$-th feature in the feature ordering of size $\depth$ given the $j$-th value in the truth table.

\begin{equation}
    \label{func:direction_func}
    \dirfunc{}(i, j) = \lfloor \frac{j-1}{2^{\depth-i}} \rfloor \bmod 2,  \quad i \in [1, \depth], j \in [1, 2^\depth]
\end{equation}


For an example $\atraindataind{q} \in \trainset$, we denote the value of the feature $\afeatind{r}$ as $\sigma(r, q)$. 
If $\dirfunc{}(i, j) = \sigma(r,q)$, it indicates that for example $\atraindataind{q}$, the feature $\afeatind{r}$ can be at the $i$-th position in the feature ordering to produce the $j$-th value in the truth table.
To classify all examples correctly, we ensure
that 
no example follows an assignment 
in the truth table leading to its opposite class. 
Thus, we propose the following constraints for classification.
Let $\atraindataind{q} \in \trainset^{+}$, for all $j = 1, \dots, 2^\depth$:
\begin{equation}
    \label{con:original_classify_pos}
    \neg \adecisionind{j} \rightarrow \bigvee_{i=1}^{\depth}\bigvee_{r=1}^{\nbfeat}(\featplacevar{r}{i} \land \dirfunc{}(i, j) \oplus \sigma(r, q))
\end{equation}

That is, for every positive example $\atraindataind{q}$, any variable $\adecisionind{j}$ assigned to $0$ must be associated to an assignment of features that contains at least one feature-value that is not coherent with $\atraindataind{q}$. 
For negative examples, we use a similar idea. 
Let $\atraindataind{q} \in \trainset^{-}$, for all $j = 1, \dots, 2^\depth$:
\begin{equation}
    \label{con:original_classify_neg}
    \adecisionind{j} \rightarrow \bigvee_{i=1}^{\depth}\bigvee_{r=1}^{\nbfeat}(\featplacevar{r}{i} \land \dirfunc{}(i, j) \oplus \sigma(r, q))
\end{equation}

\begin{table}[htb]
    \centering
    \begin{minipage}{0.72\linewidth}
        \centering
        \scalebox{0.7}{
        \footnotesize{}
        \begin{tabular}{|c||cccc||c|}
            \hline
            \textbf{$\trainset_0$} & \boldmath{$f_{1}$} & \boldmath{$f_{2}$} & \boldmath{$f_{3}$} & \boldmath{$f_{4}$} & \textbf{label} \\\hline
            $e_1$ & 1 & 0 & 1 & 0 & 0 \\\hline
            $e_2$ & 1 & 0 & 0 & 1 & 0 \\\hline
            $e_3$ & 0 & 0 & 1 & 0 & 1 \\\hline
            $e_4$ & 1 & 1 & 0 & 0 & 0 \\\hline
            $e_5$ & 0 & 0 & 0 & 1 & 1 \\\hline
            $e_6$ & 1 & 1 & 1 & 1 & 0 \\\hline
            $e_7$ & 0 & 1 & 1 & 0 & 0 \\\hline
            $e_8$ & 0 & 0 & 1 & 1 & 1 \\\hline
        \end{tabular}
        }
        \caption{\footnotesize{A dataset for binary classification}}
        \label{tab:dataset_example}
    \end{minipage}\hfill
    \begin{minipage}{0.26\linewidth}
        \centering
        \begin{tikzpicture}[scale=0.7, every node/.style={transform shape}]
    \node[roundnode]    (node1)     {$f_{1}$};
    \node[roundnode]    (node2)     [below left = 4mm and 2mm of node1]     {$f_{2}$};
    \node[squarednode]  (node3)     [below right = 4mm and 2mm of node1]                  {0};
    \node[squarednode]  (node4)     [below left= 4mm and 2mm of node2]     {1};
    \node[squarednode]  (node5)     [below right= 4mm and 2mm of node2]     {0};
    \draw [dashed]  (node1) -- (node2);
    \draw           (node1) -- (node3);
    \draw [dashed]  (node2) -- (node4);
    \draw           (node2) -- (node5);
\end{tikzpicture}
        
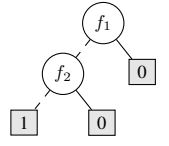
\captionof{figure}{\footnotesize{Decision Tree found}}
        \label{fig:dt_constructed}
    \end{minipage}
\end{table}

\begin{example}
\label{exp:bdd_classify_all_examples_orig}
Let $\trainset_0$ be the given set of examples shown in Table \ref{tab:dataset_example}. 
Figure \ref{fig:dt_constructed} shows the corresponding decision tree classifying all examples correctly. 
We consider 
to encode a \Bdd{} with depth $\depth=2$ classifying all examples of $\trainset_0$ correctly.
The two sets of variables are: $\{\featplacevar{1}{1}, \featplacevar{1}{2}, \featplacevar{2}{1}, \featplacevar{2}{2}, \featplacevar{3}{1}, \featplacevar{3}{2}, \featplacevar{4}{1}, \featplacevar{4}{2}\}$, and $\{\adecisionind{1}, \adecisionind{2}, \adecisionind{3}, \adecisionind{4}\}$.
The constraints \ref{con:feature_used_atmost_once}, \ref{con:exact_one_feature}, and \ref{con:root_is_bead} are:

\begin{align*}
    \centering
    \featplacevar{1}{1} + \featplacevar{1}{2} \leq 1, \quad
    \featplacevar{2}{1} + \featplacevar{2}{2} &\leq 1,  \quad
    \featplacevar{3}{1} + \featplacevar{3}{2} \leq 1, \quad
    \featplacevar{4}{1} + \featplacevar{4}{2} \leq 1 \\
    \featplacevar{1}{1} + \featplacevar{2}{1} + \featplacevar{3}{1} + \featplacevar{4}{1} &= 1, \quad
    \featplacevar{1}{2} + \featplacevar{2}{2} + \featplacevar{3}{2} + \featplacevar{4}{2} = 1 \\
    (\adecisionind{1} \oplus &\adecisionind{3}) \lor (\adecisionind{2} \oplus \adecisionind{4})
\end{align*}


For classification constraints (i.e.,~\ref{con:original_classify_pos} and~\ref{con:original_classify_neg}), 
we show the encoding of $\atraindataind{1} \in \trainset^-$ with for value $\adecisionind{1}$. The encoding for other examples and other values is similar.

\begin{equation*}
    \begin{split}
\adecisionind{1} \rightarrow (\featplacevar{1}{1} \land 0 \oplus 1) &\lor (\featplacevar{2}{1} \land 0 \oplus 0) \lor  
(\featplacevar{3}{1} \land 0 \oplus 1) \\ \lor (\featplacevar{4}{1} \land 0 \oplus 0) &\lor 
    (\featplacevar{1}{2} \land 0 \oplus 1) \lor (\featplacevar{2}{2} \land 0 \oplus 0) \\ \lor 
    (\featplacevar{3}{2} \land 0 \oplus 1) &\lor (\featplacevar{4}{2} \land 0 \oplus 0)
    \end{split}
\end{equation*}

This could be simplified as follow:
\begin{equation*}
    \neg \adecisionind{1} \lor \featplacevar{1}{1} \lor
    \featplacevar{3}{1} \lor \featplacevar{1}{2} \lor \featplacevar{3}{2}
\end{equation*}

\begin{table}[htb]
    \centering
    \begin{minipage}{0.59\linewidth}
        \centering
        \scalebox{0.7}{
        \begin{tabular}{|cc||c|}
            \hline
            $\avarind{1}=\afeatind{1}$ & $\avarind{2}=\afeatind{2}$ & \\\hline
            $0$ & $0$ & $\adecisionind{1}=1$ \\\hline
            $0$ & $1$ & $\adecisionind{2}=0$ \\\hline
            $1$ & $0$ & $\adecisionind{3}=0$ \\\hline
            $1$ & $1$ & $\adecisionind{4}=0$ \\\hline
        \end{tabular}
        }
        \caption{\footnotesize{Truth table solution for \texttt{BDD} of depth 2 classifying all example of $\trainset_0$}}
        \label{tab:Bdd-solution}
    \end{minipage}\hfill
    \begin{minipage}{0.39\linewidth}
        \centering
        \begin{tikzpicture}[scale=0.7, every node/.style={transform shape}]
    \node[roundnode]    (node1)     {$f_{1}$};
    \node[roundnode]    (node2)     [below left = 3mm and 5mm of node1]     {$f_{2}$};
    \node[squarednode]  (node3)     [below = 10mm of node1]                  {0};
    \node[squarednode]  (node4)     [below left= 3mm and 5mm of node2]     {1};
    \draw [dashed]  (node1) -- (node2);
    \draw           (node1) -- (node3);
    \draw [dashed]  (node2) -- (node4);
    \draw           (node2) -- (node3);
\end{tikzpicture}
        \captionof{figure}{\footnotesize{The \texttt{BDD} found}}
        \label{fig:bdd_constructed}
    \end{minipage}
\end{table}
The values of truth table found by the \Sat{} model are shown in Table~\ref{tab:Bdd-solution}, the feature ordering is $[\afeatind{1}, \afeatind{2}]$. 
Moreover, Table \ref{tab:Bdd-solution} illustrates the relationship between the values of truth table and the assignments of the given sequence Boolean variable of size $2$.
Figure \ref{fig:bdd_constructed} shows 
the corresponding \Bdd{}.
This \Bdd{} classifies all examples of the dataset $\trainset_0$ correctly,
also provides more compact representation than the decision tree shown in Figure \ref{fig:dt_constructed}.

\end{example}


We refer to this first \Sat\ encoding for $P_{bdd}(\trainset, \depth)$ as \Bdd1. The size of \Bdd1 is given in Proposition~\ref{prop:model_size_bdd_orig}.

\begin{proposition}
\label{prop:model_size_bdd_orig}
For a $P_{bdd}(\trainset, \depth)$ problem with $\nbfeat$ binary features and $\nbtraindata{}$ examples, the encoding size (in terms of the number of literals used in the different clauses) of \Bdd1\ is $O(\nbtraindata{} \times \depth \times \nbfeat \times 2^\depth)$.  
\end{proposition}
\begin{proof}
Notice first that $j$ ranges from $1$ to $2^\depth$, $i$ ranges from $1$ to $\depth$, and $r$ ranges from $1$ to $\nbfeat$. 
The term $\nbtraindata{} \times 2^\depth$ results from constraint (\ref{con:original_classify_pos}) and (\ref{con:original_classify_neg}), each contains $O(\depth \times \nbfeat)$ literals.
For the remaining constraints, it is $O(\depth \times \nbfeat)$ for constraints~(\ref{con:feature_used_atmost_once}) 
and~(\ref{con:exact_one_feature}), $O(2^\depth)$ for constraint~(\ref{con:root_is_bead}).
\end{proof}

The size of \Bdd1 is quite huge due to the size of clauses
generated by constraints (\ref{con:original_classify_pos}) and  (\ref{con:original_classify_neg}) for classification.
This makes \Bdd1 impractical in practice. 

\subsection{An improved SAT Model for $P_{{bdd}}(\trainset, \depth)$ }


In order to reduce the size of \Bdd{1}, we propose new classification constraints to replace constraints (\ref{con:original_classify_pos}) and (\ref{con:original_classify_neg}).
The idea is that every positive (respectivery negative) example follows an assignment leading to a positive (respectively negative) value of the truth table.
We introduce a new set of Boolean
variables:

\begin{itemize}
    \item 
    $\signfeatexample{i}{q}$: The variable $\signfeatexample{i}{q}$ is 1 iff for example $\atraindataind{q}$ the value of the
    $i$-th feature selected in feature ordering is 1, where $i=1,\dots,\depth$, $q=1, \dots, \nbtraindata{}$.
\end{itemize}

Then, We describe constraints that relate the values
of features for each example $\atraindataind{q} \in \trainset$, for $i=1,\dots, \depth$, $r=1,\dots, \nbfeat$:
\begin{equation}
    \label{con:sign_features}
    \begin{split}
        \featplacevar{r}{i} &\rightarrow \signfeatexample{i}{q} \quad\;\;\, \text{if } \sigma(q, r) = 1 \\
        \featplacevar{r}{i} &\rightarrow \neg \signfeatexample{i}{q} \quad \text{if } \sigma(q, r) = 0
    \end{split}
\end{equation}

Let $\atraindataind{q} \in \trainset^+$, we have $2^\depth$ constraints for classifying examples correctly:
\begin{equation}
    \label{con:improved_classify_pos}
    \begin{split}
        \neg \signfeatexample{1}{q} \land \neg \signfeatexample{2}{q} \land &\dots \land \neg \signfeatexample{\depth-1}{q} \land \neg \signfeatexample{\depth}{q} \rightarrow \adecisionind{1} \\
        \neg \signfeatexample{1}{q} \land \neg \signfeatexample{2}{q} \land &\dots \land \neg \signfeatexample{\depth-1}{q} \land \signfeatexample{\depth}{q} \rightarrow \adecisionind{2} \\
        &\dots \\
         \signfeatexample{1}{q} \land \signfeatexample{2}{q} \land &\dots \land \signfeatexample{\depth-1}{q} \land \signfeatexample{\depth}{q} \rightarrow \adecisionind{2^\depth} 
    \end{split}
\end{equation}

That is, any positive example follows an assignment of the feature ordering that leads to a positive value in the truth table.

Similarly, for any $\atraindataind{q} \in \trainset^-$, we also have $2^\depth$ constraints: 
\begin{equation}
    \label{con:improved_classify_neg}
    \begin{split}
        \neg \signfeatexample{1}{q} \land \neg \signfeatexample{2}{q} \land &\dots \land \neg \signfeatexample{\depth-1}{q} \land \neg \signfeatexample{\depth}{q} \rightarrow \neg \adecisionind{1} \\
        \neg \signfeatexample{1}{q} \land \neg \signfeatexample{2}{q} \land &\dots \land \neg \signfeatexample{\depth-1}{q} \land \signfeatexample{\depth}{q} \rightarrow \neg \adecisionind{2} \\
        &\dots \\
         \signfeatexample{1}{q} \land \signfeatexample{2}{q} \land &\dots \land  \signfeatexample{\depth-1}{q} \land \signfeatexample{\depth}{q} \rightarrow \neg \adecisionind{2^\depth} 
    \end{split}
\end{equation}


We refer to this new \Sat\ encoding for $P_{bdd}(\trainset, \depth)$ as \Bdd2. The encoding size of \Bdd2 is given in Proposition~\ref{prop:model_size_bdd_improved}.

\begin{proposition}
\label{prop:model_size_bdd_improved}
For a $P_{dd}(\trainset, \depth)$ problem with $\nbfeat$ binary features and $\nbtraindata{}$ examples, the encoding size of the SAT encoding (\Bdd2) is $O(\nbtraindata{} \times \depth \times (2^\depth +\nbfeat))$. 
\end{proposition}

\begin{proof}
The term $\nbtraindata{} \times \depth \times \nbfeat$ results from constraint (\ref{con:sign_features}).
For constraints~(\ref{con:improved_classify_pos}) 
and~(\ref{con:improved_classify_neg}), for each example, there are $2^\depth$ clauses containing $\depth + 1$ literals. 
The term $\nbtraindata{} \times \depth \times 2^\depth$ results from that.
\end{proof}

Propositions \ref{prop:model_size_bdd_orig} and \ref{prop:model_size_bdd_improved} show a clear theoretical advantage of \Bdd{2} compared to \Bdd{1} in terms of the encoding size, thus scalability. 

\subsection{MaxSAT Model for $P_{{bdd}}^*(\trainset, \depth):$}
We now present a \MaxSat\ encoding for the optimization problem $P_{bdd}^*(\trainset, \depth)$.
That is, given a set of examples $\trainset$, find a binary decision diagram of depth $\depth$ that maximises the number of examples correctly classified.


We transform the \Sat\ encoding of \Bdd{}s into a \MaxSat\ encoding following a simple technique. 
The idea is to keep structural constraints as hard clauses and classification constraints as soft clauses. 
We consider 
\Bdd2 as it has a reduced size.
Constraints (\ref{con:feature_used_atmost_once}),  (\ref{con:exact_one_feature}), (\ref{con:root_is_bead}) and~(\ref{con:sign_features}) are kept as hard clauses.
To classify the examples, we declare all clauses of constraints~(\ref{con:improved_classify_pos}) 
and~(\ref{con:improved_classify_neg}) as soft clauses.
For any example $\atraindataind{q}$, the number of satisfied soft clauses associated to $\atraindataind{q}$ is either $2^\depth$ (indicating $\atraindataind{q}$ is classified correctly), or $2^\depth-1$ (indicating $\atraindataind{q}$ is classified wrongly). 
Therefore, the objective of maximising the number of satisfied soft clauses is equivalent to maximise the number of examples correctly classified.

\subsection{Merging Compatible Subtrees}


Consider a \Bdd\ $\BddGraph{}$ found by a \MaxSat\ solver and its associated truth table $\Truthtable{}$.
Based on the feature ordering of $\BddGraph{}$, it is possible that some values in $\Truthtable{}$ capture no (training) example 
(Equivalent to ``\textit{unknown}'' nodes for \OODG{}).
Such values are decided by the \MaxSat\ solver in an arbitrary way, which gives a certain bias in generalisation.
We propose to merge compatible subtrees in $\BddGraph{}$ in order to handle this bias. 
This will result in changing some values in the truth table $\Truthtable{}$ (i.e. the arbitrary ones decided by \MaxSat).

We propose a post-processing procedure to merge compatible subtrees using the following three phase: 
(1) update the truth table $\Truthtable{}$ by replacing the values of $\Truthtable{}$ that capture no examples
with a special value ``u'';
(2) for each level, check the \Beads{}, where ``u'' can be used to match $1$ or $0$, and create a node for each \Bead{};
(3) for each level, after creating the nodes, check the matches between all subtables of the next level.  
For matched subtables, update the corresponding \Beads{} of current level to eliminate the ``u'' values.
This process is illustrated in Example~\ref{exp:compatible_merge}.

\begin{example}
\label{exp:compatible_merge}
Assume that a \MaxSat\ model finds the truth table $\Truthtable{}$ $00010111$ with the feature ordering $[f_1, f_2, f_3]$,
and assume that the updated truth table $\Truthtable{'}$ is $u0u1011u$ after phase (1).
For level $1$, we create a root node for $\Truthtable{'}$ as it is a bead.
Then, we check the subtables of $\Truthtable{'}$ ($u0u1$ and $011u$). Since they do not match, we move to the next level.
For level $2$, we create a node for $u0u1$ and a node for $011u$ as they are beads.
Then, we check all subtables of the next level, which are $u0$, $u1$, $01$ and $1u$.
We observe that $u0$ matches $1u$, and $u1$ matches $01$, therefore, the beads $u0u1$ and $011u$ are updated as $1001$ and $0110$.
Therefore, the updated beads of $\Truthtable{'}$ are $\{u0u1011u, 1001(u0u1), 0110(011u), 10, 01, 0, 1\}$. 

Figure \ref{fig:Ex3_compatible_bdd} shows the beads updated (in the left), the \Bdd{} before the merging process (in the right) and the final \Bdd{} found (in the middle).
\end{example}
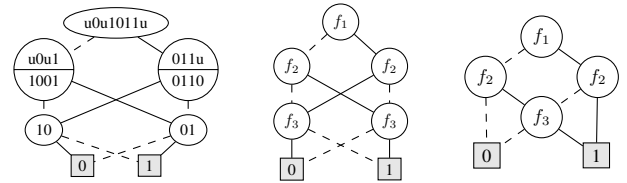
\begin{figure}[htb]
    \centering
    \begin{minipage}{0.365\linewidth}
        \centering
                \begin{tikzpicture}[scale=0.6, every node/.style={transform shape}]
        \node[beadnodeone]    (node1)     {u0u1011u};
        \node[beadnodethree]    (node2)     [below left = 3mm and 3mm of node1]     {u0u1 \nodepart{lower} 1001};
        \node[beadnodethree]    (node3)     [below right=3mm and 3mm of node1]     {011u \nodepart{lower} 0110};
        \node[beadnodetwo]    (node4)     [below = 3mm of node2]    {10};
        \node[beadnodetwo]    (node5)     [below = 3mm of node3]    {01};
        \node[squarednode]  (node6)     [below right = 3mm and 3mm of node4]     {0};
        \node[squarednode]  (node7)     [below left = 3mm and 3mm of node5]     {1};
            \draw           (node1) -- (node3);
            \draw [dashed]  (node1) -- (node2);
            \draw           (node2) -- (node5);
            \draw [dashed]  (node2) -- (node4);
            \draw           (node3) -- (node4);
            \draw [dashed]  (node3) -- (node5);
            \draw           (node4) -- (node6);
            \draw [dashed]  (node4) -- (node7);
            \draw           (node5) -- (node7);
            \draw [dashed]  (node5) -- (node6);
        \end{tikzpicture}
    \end{minipage} \hfill
    \begin{minipage}{0.32\linewidth}
        \centering
        \begin{tikzpicture}[scale=0.61, every node/.style={transform shape}]
            \node[roundnode]    (node1)     {$f_1$};
            \node[roundnode]    (node2)     [below left = 4mm and 5mm of node1]     {$f_2$};
            \node[roundnode]    (node3)     [below right= 4mm and 5mm of node1]     {$f_2$};
            \node[roundnode]    (node4)     [below = 4mm of node2]     {$f_3$};
            \node[roundnode]    (node5)     [below = 4mm of node3]     {$f_3$};
            \node[squarednode]  (node6)     [below = 4mm of node4]     {0};
            \node[squarednode]  (node7)     [below = 4mm of node5]     {1};
            \draw           (node1) -- (node3);
            \draw [dashed]  (node1) -- (node2);
            \draw           (node2) -- (node5);
            \draw [dashed]  (node2) -- (node4);
            \draw           (node3) -- (node4);
            \draw [dashed]  (node3) -- (node5);
            \draw           (node4) -- (node6);
            \draw [dashed]  (node4) -- (node7);
            \draw           (node5) -- (node7);
            \draw [dashed]  (node5) -- (node6);
        \end{tikzpicture}
    \end{minipage}
    \hfill
    \begin{minipage}{0.29\linewidth}
        \centering
        \begin{tikzpicture}[scale=0.7, every node/.style={transform shape}]
            \node[roundnode]    (node1)     {$\afeatind{1}$};
            \node[roundnode]    (node2)     [below left = 2mm and 5mm of node1]     {$\afeatind{2}$};
            \node[roundnode]    (node3)     [below right= 2mm and 5mm of node1]     {$\afeatind{2}$};
            \node[roundnode]    (node4)     [below left = 2mm and 5mm of node3]     {$\afeatind{3}$};
            \node[squarednode]  (node5)     [below left = 2mm and 5mm of node4]     {0};
            \node[squarednode]  (node6)     [below right= 2mm and 5mm of node4]     {1};
            \draw           (node1) -- (node3);
            \draw [dashed]  (node1) -- (node2);
            \draw           (node2) -- (node4);
            \draw [dashed]  (node2) -- (node5);
            \draw           (node3) -- (node6);
            \draw [dashed]  (node3) -- (node4);
            \draw           (node4) -- (node6);
            \draw [dashed]  (node4) -- (node5);
        \end{tikzpicture}
    \end{minipage}
    \caption{\footnotesize{The \Bdd{} after merging compatible subtrees (the middle one) and the \Bdd{} before merging (the right one).}}
    \label{fig:Ex3_compatible_bdd}
\end{figure}

\section{Experimental Results}
\label{sec:exp}

We present our large experimental studies to evaluate empirically our propositions on different levels\footnote{The source code and datasets are available online at https://gitlab.laas.fr/hhu/bddencoding}.
At first, we make some preliminary experiments 
on the proposed \Sat{} 
models to confirm the great improvements in the encoding size of \Bdd2 compared to \Bdd1, as shown theoretically in proposition \ref{prop:model_size_bdd_orig} and \ref{prop:model_size_bdd_improved}.
Then, we evaluate the prediction performance between the proposed \MaxSat-\Bdd{} model and the heuristic methods, \ODT{} and \OODG{}~\cite{DBLP:conf/ijcai/KohaviL95}.
Next, we compare our \MaxSat-\Bdd{} model 
with an exact method for building decision trees using \MaxSat{}~\cite{DBLP:conf/ijcai/Hu0HH20} in {terms} of prediction quality, model size, and encoding size.
Finally, we propose and evaluate
a simple, yet efficient, scalable heuristic version of our \MaxSat-\Bdd{} model.

We consider
datasets from CP4IM\footnote{https://dtai.cs.kuleuven.be/CP4IM/datasets/}. 
These datasets are binarized with the one-hot encoding.
Table \ref{tab:dataset_exp_cp4im} describes 
the characteristics of these datasets:
$\nbtraindata$ indicates the number of examples,
$\nbfeat_{orig}$ indicates the original number of features, $\nbfeat$ indicates the number of binary features after binarization, 
and $pos$ indicates the percentage of positive examples.

All experiments were run on a cluster using Xeon E5-2695 v3@2.30GHz CPU running xUbuntu 16.04.6 LTS.
The \Sat{} solver we used is Kissat~\cite{BiereFazekasFleuryHeisinger-SAT-Competition-2020-solvers}, the winner of \Sat{} competition 2020.
For each experiment of \Sat{} encoding, we set $20$ hours as the global timeout for \Sat{} solver.
The \MaxSat\ solver we used is Loandra~\cite{DBLP:conf/cpaior/BergDS19},
an efficient incomplete \MaxSat\ solver. 
For each experiment of \MaxSat-\Bdd{}, the time limit for generating formulas and the time limit for solver are set to $15$ minutes.

\begin{table}[htb]
    \centering
    \renewcommand{\arraystretch}{0.85}
    \setlength{\tabcolsep}{10pt}
    \small{}
        \begin{tabular}{||c||c|c|c||c||}
            \hline
            \hline
            \textbf{Dataset} & $\nbtraindata$ & $\nbfeat{}_{orig}$ & $\nbfeat{}$ & $pos$\\ \hline
            \hline
            anneal & 812 & 42 & 89 & 0.77\\\hline
            audiology & 216 & 67 & 146 & 0.26 \\\hline
            australian & 653 & 51 & 124 & 0.55 \\\hline
            cancer & 683 & 9 & 89 & 0.35\\\hline
            car & 1728 & 6 & 21 & 0.30 \\\hline
            cleveland & 296 & 45 & 95 & 0.54\\\hline
            hypothyroid & 3247 & 43 & 86 & 0.91\\\hline
            kr-vs-kp & 3196 & 36 & 73 & 0.52\\\hline
            lymph & 148 & 27 & 68 & 0.55\\\hline
            mushroom & 8124 & 21 & 112 & 0.52\\\hline
            tumor & 336 & 15 & 31 & 0.24\\\hline
            soybean & 630 & 16 & 50 & 0.15 \\\hline
            splice-1 & 3190 & 60 & 287 & 0.52\\\hline
            tic-tac-toe & 958 & 9 & 27 & 0.65\\\hline
            vote & 435 & 16 & 48 & 0.61\\\hline
            \hline
        \end{tabular}
        \caption{\footnotesize{Detailed Information Regarding the Datasets}}
        \label{tab:dataset_exp_cp4im}
\end{table}

\subsection{Comparison Of The SAT Encodings}

We consider the optimisation problem of finding a \Bdd{} that classifies all training examples correctly with the minimum depth.
We use a simple linear search by solving multiple times the decision problem asking to find a \Bdd{} with a given depth $\depth$ (Problem $P_{bdd}(\trainset, \depth)$ in Section \ref{sec:Sat-models}).
We set the initial depth $\depth_0=7$.
Considering the scalability problem, for each dataset, we use the hold-out method to split the training and testing set.
We choose $5$ different small splitting ratios $r=\{0.05, 0.1, 0.15, 0.2, 0.25\}$ to generate the training set. The remaining instances are used for testing.
This process is repeated $10$ times with different random seeds.

Table \ref{tab:annex1_satbdd_diff} reports the average results of instances that are solved to optimality by all methods within the given time. 
The columns ``Acc'' and ``dopt'' indicate the average testing accuracy in percent and the average optimal depth, respectively. 
The encoding size in given in column ``E\_Size'' (i.e., the number of literals in the cnf file divided by $10^3$). The column ``Time'' indicates the runtime in seconds of successful runs.
The value ``N/A'' indicates the lack of results because of the timeout.
The best values are indicated in blue.

Table \ref{tab:annex1_satbdd_diff} shows the great improvements in terms of the encoding size and the runtime of \Bdd{2} compared to \Bdd{1}. 
This empirical observation is coherent with the complexity analysis made in Proposition \ref{prop:model_size_bdd_orig} and \ref{prop:model_size_bdd_improved}.
We also observe no substantial difference in terms of testing accuracy between the two approaches. 

\begin{table}[htb!]
    \centering
    \renewcommand{\arraystretch}{0.8}
    \setlength{\tabcolsep}{1.3pt}
    \scriptsize{}
    \begin{tabular}{|c|c||c|c|c|c||c|c|c|c||}
    \hline
    \multirow{2}{*}{\scriptsize{\textbf{Datasets}}} & 
    \multirow{2}{*}{\scriptsize{\textbf{ratio}}} &
    \multicolumn{4}{c||}{\scriptsize{\textbf{BDD1}}} & 
    \multicolumn{4}{c||}{\scriptsize{\textbf{BDD2}}} \\
    \cline{3-10} 
            & &
            \textbf{Acc} & \textbf{dopt} & \textbf{E\_Size} & \textbf{Time} &
            \textbf{Acc} & \textbf{dopt} & \textbf{E\_Size} & \textbf{Time}\\ 
    \hline
    \multirow{1}{*}{anneal}
 & 0.05 & N/A  & N/A  & N/A & N/A  & \cellcolor{blue!50}68.65 & 6.33 & \cellcolor{blue!50}9.43 & \cellcolor{blue!50}192.74\\
\hline

\multirow{5}{*}{audiology}
 & 0.05 & 75.99 & 2 & \cellcolor{blue!50}0.83 & 0.46 & \cellcolor{blue!50}76.86 & 2 & \cellcolor{blue!50}0.83 & \cellcolor{blue!50}0.07\\
 & 0.10 & \cellcolor{blue!50}91.22 & 2.50 & 2.86 & 0.84 & 90.97 & 2.50 & \cellcolor{blue!50}2.01 & \cellcolor{blue!50}0.07\\
 & 0.15 & 92.54 & 2.80 & 5.37 & 1.39 & \cellcolor{blue!50}93.3 & 2.80 & \cellcolor{blue!50}3.26 & \cellcolor{blue!50}0.09\\
 & 0.20 & \cellcolor{blue!50}90.46 & 3.10 & 11.54 & 4.65 & 90 & 3.10 & \cellcolor{blue!50}4.7 & \cellcolor{blue!50}0.17\\
 & 0.25 & \cellcolor{blue!50}92.94 & 3.60 & 23.53 & 31.79 & 92.45 & 3.60 & \cellcolor{blue!50}6.81 & \cellcolor{blue!50}0.39\\
\hline

\multirow{2}{*}{australian}
 & 0.05 & \cellcolor{blue!50}80.21 & 3.30 & 8.54 & 21.65 & 79.02 & 3.30 & \cellcolor{blue!50}3.36 & \cellcolor{blue!50}0.54\\
 & 0.10 & N/A  & N/A  & N/A & N/A  & \cellcolor{blue!50}77.46 & 6.22 & \cellcolor{blue!50}15.21 & \cellcolor{blue!50}7473.32\\
\hline

\multirow{5}{*}{cancer}
 & 0.05 & \cellcolor{blue!50}86.97 & 2.70 & 3.55 & 0.49 & 86.69 & 2.70 & \cellcolor{blue!50}2.04 & \cellcolor{blue!50}0.08\\
 & 0.10 & \cellcolor{blue!50}89.97 & 4 & 23.15 & 4.95 & 89.24 & 4 & \cellcolor{blue!50}6.01 & \cellcolor{blue!50}0.55\\
 & 0.15 & \cellcolor{blue!50}90.7 & 5.20 & 104.81 & 156.32 & 90.29 & 5.20 & \cellcolor{blue!50}12.93 & \cellcolor{blue!50}3.72\\
 & 0.20 & 91.53 & 6.40 & 390.71 & 10224.45 & \cellcolor{blue!50}91.57 & 6.40 & \cellcolor{blue!50}26.5 & \cellcolor{blue!50}55.2\\
 & 0.25 & N/A  & N/A  & N/A & N/A  & \cellcolor{blue!50}92.16 & 6.44 & \cellcolor{blue!50}35.12 & \cellcolor{blue!50}200.12\\
\hline

\multirow{1}{*}{car}
 & 0.05 & 76.81 & 7.62 & 162.38 & 5538.01 & \cellcolor{blue!50}80.18 & 7.62 & \cellcolor{blue!50}19.42 & \cellcolor{blue!50}1924.9\\
\hline

\multirow{5}{*}{cleveland}
 & 0.05 & \cellcolor{blue!50}68.19 & 2.50 & 1.41 & 0.86 & 64.72 & 2.50 & \cellcolor{blue!50}1.03 & \cellcolor{blue!50}0.07\\
 & 0.10 & 68.58 & 3.80 & 9.15 & 121.11 & \cellcolor{blue!50}69.29 & 3.80 & \cellcolor{blue!50}2.94 & \cellcolor{blue!50}0.92\\
 & 0.15 & \cellcolor{blue!50}72.53 & 4.80 & 27.65 & 800.35 & 70.83 & 4.80 & \cellcolor{blue!50}5.63 & \cellcolor{blue!50}15.07\\
 & 0.20 & N/A  & N/A  & N/A & N/A  & \cellcolor{blue!50}68.78 & 6.10 & \cellcolor{blue!50}10.9 & \cellcolor{blue!50}2616.23\\
 & 0.25 & N/A  & N/A  & N/A & N/A  & \cellcolor{blue!50}68.74 & 6.90 & \cellcolor{blue!50}18.13 & \cellcolor{blue!50}16405.46\\
\hline

\multirow{1}{*}{hypothyroid}
 & 0.05 & 96.26 & 5 & 131.80 & 319.09 & \cellcolor{blue!50}96.3 & 5 & \cellcolor{blue!50}18.22 & \cellcolor{blue!50}2.98\\
\hline

\multirow{5}{*}{lymph}
 & 0.05 & 67.23 & 2 & \cellcolor{blue!50}0.33 & 0.13 & \cellcolor{blue!50}68.79 & 2 & 0.34 & \cellcolor{blue!50}0.07\\
 & 0.10 & 67.69 & 2.60 & 1.11 & 0.46 & \cellcolor{blue!50}70.37 & 2.60 & \cellcolor{blue!50}0.79 & \cellcolor{blue!50}0.07\\
 & 0.15 & 70.16 & 3.40 & 3.41 & 2.83 & \cellcolor{blue!50}71.98 & 3.40 & \cellcolor{blue!50}1.52 & \cellcolor{blue!50}0.16\\
 & 0.20 & 70.34 & 3.90 & 6.98 & 21.77 & \cellcolor{blue!50}72.94 & 3.90 & \cellcolor{blue!50}2.24 & \cellcolor{blue!50}0.62\\
 & 0.25 & \cellcolor{blue!50}72.23 & 5.10 & 23.08 & 410.10 & 68.48 & 5.10 & \cellcolor{blue!50}3.91 & \cellcolor{blue!50}4.31\\
\hline

\multirow{5}{*}{mushroom}
 & 0.05 & \cellcolor{blue!50}99.55 & 5.30 & 518.03 & 3774.19 & 99.48 & 5.30 & \cellcolor{blue!50}60.03 & \cellcolor{blue!50}28.32\\
 & 0.10 & 99.81 & 5.80 & 1523.69 & 12552.47 & \cellcolor{blue!50}99.87 & 5.80 & \cellcolor{blue!50}138.73 & \cellcolor{blue!50}116.07\\
 & 0.15 & \cellcolor{blue!50}99.87 & 5.90 & 2475.01 & 20616.09 & 99.86 & 5.90 & \cellcolor{blue!50}213.62 & \cellcolor{blue!50}210.98\\
 & 0.20 & \cellcolor{blue!50}99.97 & 6 & 3503.87 & 29342.71 & 99.92 & 6 & \cellcolor{blue!50}292.25 & \cellcolor{blue!50}278.23\\
 & 0.25 & \cellcolor{blue!50}99.96 & 6 & 4381.69 & 33787.52 & 99.95 & 6 & \cellcolor{blue!50}365.18 & \cellcolor{blue!50}437.4\\
\hline


\multirow{2}{*}{soybean}
 & 0.05 & \cellcolor{blue!50}80.52 & 3.90 & 5.74 & 2.94 & 78.58 & 3.90 & \cellcolor{blue!50}1.78 & \cellcolor{blue!50}0.21\\
 & 0.10 & \cellcolor{blue!50}84.47 & 5.90 & 70.84 & 1391.79 & 82.22 & 5.90 & \cellcolor{blue!50}7.48 & \cellcolor{blue!50}7.46\\
\hline

\multirow{2}{*}{tic-tac-toe}
 & 0.05 & 64.46 & 5.90 & 23.57 & 41.18 & \cellcolor{blue!50}66.37 & 5.90 & \cellcolor{blue!50}3.83 & \cellcolor{blue!50}5.62\\
 & 0.10 & 72.68 & 7.60 & 206.10 & 18589.30 & \cellcolor{blue!50}74.61 & 7.60 & \cellcolor{blue!50}21.52 & \cellcolor{blue!50}2846.24\\
\hline

\multirow{5}{*}{vote}
 & 0.05 & 90.89 & 2.10 & 0.61 & 0.27 & \cellcolor{blue!50}91.09 & 2.10 & \cellcolor{blue!50}0.57 & \cellcolor{blue!50}0.08\\
 & 0.10 & \cellcolor{blue!50}91.93 & 2.60 & 1.99 & 0.57 & 91.63 & 2.60 & \cellcolor{blue!50}1.34 & \cellcolor{blue!50}0.08\\
 & 0.15 & \cellcolor{blue!50}92.27 & 3.40 & 6.97 & 1.20 & \cellcolor{blue!50}92.27 & 3.40 & \cellcolor{blue!50}2.75 & \cellcolor{blue!50}0.16\\
 & 0.20 & 92.35 & 4.10 & 20.68 & 26.43 & \cellcolor{blue!50}92.44 & 4.10 & \cellcolor{blue!50}4.76 & \cellcolor{blue!50}0.71\\
 & 0.25 & \cellcolor{blue!50}93 & 4.90 & 53.92 & 348.56 & 92.42 & 4.90 & \cellcolor{blue!50}8.13 & \cellcolor{blue!50}3.4\\
\hline
    \end{tabular}
    \caption{\label{tab:annex1_satbdd_diff}
    \footnotesize{Evaluation of the Reduced SAT Model}
    }
\end{table}


\begin{figure*}[htb!]
        \centering
        \includegraphics[width=.8\linewidth, height=.28\linewidth]{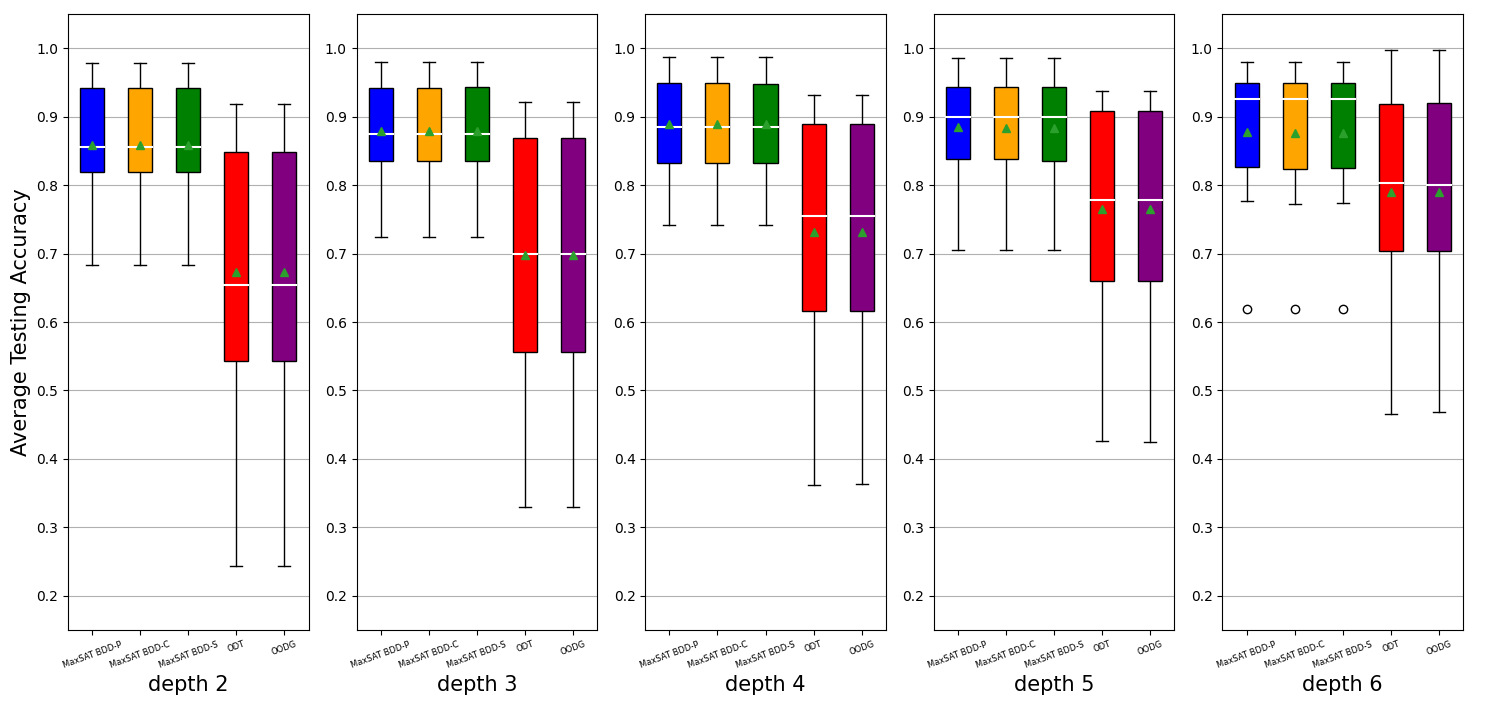}
        \captionof{figure}{
        \footnotesize{
        The average testing accuracy with different biases: \textbf{MaxSAT BDD-P}, \textbf{MaxSAT BDD-C}, \textbf{MaxSAT BDD-S}, \ODT{}, \OODG{} (respectively from left to right).
        }}
        \label{fig:exp2_oodg_maxsatbdd_all}
\end{figure*}

\subsection{Comparison with Existing Heuristic Approaches}

We consider the proposed \MaxSat-\Bdd{} model for solving the $P_{{bdd}}^*(\trainset, \depth)$ problem (defined in Section~\ref{sec:problemdefinition}) with $5$ different depths $\depth \in \{2, 3, 4, 5, 6\}$.
For each dataset, we use random $5$-fold cross-validation with $5$ different seeds.
We compare our \MaxSat-\Bdd{} model with the heuristic approaches proposed in \cite{DBLP:conf/ijcai/KohaviL95} to learn \ODT{} and \OODG{}.
For the heuristic methods, as described in the background section, after merging the \textit{isomorphic} and \textit{compatible} subtrees of \ODT{}, 
the 
\OODG{} 
changes the bias for those ``\textit{unknown}'' nodes.
In fact, different bias affects the prediction for unseen examples, but \textit{not} for the training examples.
Therefore, the training accuracies of \ODT{}  and \OODG{} are equal 
whereas the testing accuracies could be different.
This fact is also true in the \MaxSat-\Bdd{} model.
In this experiment, we consider the following three biases:

 \begin{itemize}
     
\item By assigning for each unknown node the majority class of its branch
(denoted as \textbf{MaxSAT BDD-P})
\item By merging compatible subtrees (\textbf{MaxSAT BDD-C})
\item By using the class decided by the \MaxSat\ solver (\textbf{MaxSAT BDD-S}).
 \end{itemize}
 
\begin{figure}[h!]
        \centering
        \includegraphics[scale=.18]  
        {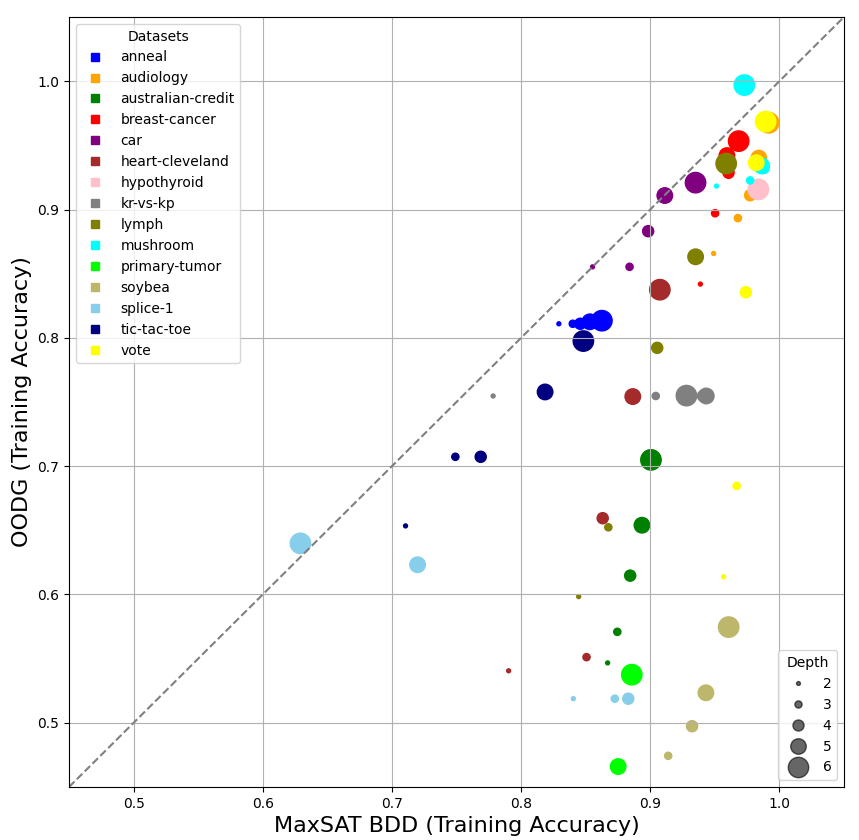}
        \captionof{figure}{\footnotesize{Comparison between the average training accuracy of \OODG{} and the  \MaxSat-\Bdd{}}}
        \label{fig:OODGvsMaxSAT}
\end{figure}

Figure~\ref{fig:OODGvsMaxSAT} presents the comparison of the average training accuracy between \OODG{} and \MaxSat-\Bdd{} model.
In this figure, 
different datasets are marked with different colors, and different depths are labelled with points of different sizes.
From the scatter {plot}, we observe that the average training accuracy of both approaches increase with the increase of depth. 
{Moreover, and more importantly,} the \MaxSat-\Bdd{} model performs better than the {heuristic} \OODG{} in training accuracy. 

Figure~\ref{fig:exp2_oodg_maxsatbdd_all} shows the average testing accuracy of \MaxSat-\Bdd{} with different biases, \ODT{}, and \OODG{}  using different depths averaged over all datasets.
The white line and green triangle of each box indicate the median and the average value, {respectively}.
Clearly,
the \MaxSat-\Bdd{} models have better prediction performance than \ODT{}  and \OODG{}. This is particularly true with small depths.
Increasing the depth increases the predictions for all methods as expected.
However, the increase is less important with the different \MaxSat-\Bdd{} models. 
We observe also that there is little difference between the different biases for \MaxSat-\Bdd{}. This suggests that the optimal solutions are somewhat robust to the bias.
We noticed also that for all datasets (except one), all \MaxSat-\Bdd{} models report optimality when the depth is equal to $2$. 

\subsection{Comparison with an Exact Decision Tree Approach}
The purpose of this experiment is to 
{compare} our proposition
with the exact method for learning decision trees using the same solving approach (\MaxSat).
For \MaxSat-\Bdd{},
we consider only the bias of merging compatible subtrees \textbf{MaxSAT BDD-C} since no substantial difference was observed between the different biases.
We consider different values for the depth: $\depth \in \{2, 3, 4, 5, 6\}$.
For each dataset, we use random $5$-fold cross-validation with $5$ different seeds.
For \MaxSat-\Bdd{}, the depth also corresponds to the number of selected features,
whereas for \MaxSat-\DT{}
the depth indicates the \textit{maximum depth} of the \Bdd. 
Table \ref{tab:exp3_maxsatbdd_maxsatdt_all} 
presents the results of the evaluation.
The column ``Size'' and ``E\_Size'' indicate the number of nodes of the model and the encoding size (number of literals divided by $10^3$), respectively.
The column ``F\_d'' indicates the average number of features used in the decision tree.
The best values are indicated in blue.

The results in Table \ref{tab:exp3_maxsatbdd_maxsatdt_all}
show that the \MaxSat-\Bdd{} approach is competitive to \MaxSat-\DT{} in terms of prediction quality.
In most cases, the training and testing accuracy of these two approaches are close.
However, the size of the models are always smaller with \MaxSat-\Bdd{}. 
The difference grows bigger when the depth increases.
The reduction in model size provides better intrepretability.
Moreover, sometimes, compared to the optimal \Bdd{}s found via \MaxSat-\Bdd{}, the optimal decision trees found via \MaxSat-\DT{} uses useless splits.
This is, for instance, the case for the datasets ``\textit{car}'' and ``\textit{hypothyroid}'' with depth $2$.
We observe also that \MaxSat-\Bdd{} has always a {much} lighter encoding size than \MaxSat-\DT{}. This gives a clear advantage to \MaxSat-\Bdd{} to handle the problem and to report optimality. 

\begin{table}[!h]
    \centering
    \renewcommand{\arraystretch}{0.7}
    \setlength{\tabcolsep}{1.6pt}
    \scriptsize{}
    \begin{tabular}{|c|c||c|c|c|c||c|c|c|c|c||}
    \hline
    \multirow{2}{*}{\scriptsize{\textbf{Datasets}}} & 
    \multirow{2}{*}{\scriptsize{\textbf{\depth{}}}} &
    \multicolumn{4}{c||}{\scriptsize{\textbf{MaxSAT BDD-C}}} & 
    \multicolumn{5}{c||}{\scriptsize{\textbf{MaxSAT-DT}}} \\
    \cline{3-11} 
            & &
            \textbf{Train} &
            \textbf{Test} &
            \textbf{Size} &
            \textbf{E\_Size} 
            &\textbf{Train} &
            \textbf{Test} & 
            \textbf{Size} &
            \textbf{E\_Size}&
            \textbf{F\_d}\\ 
    \hline
    \multirow{5}{*}{\shortstack[1]{anneal}}
 & 2 & 82.92 & \cellcolor{blue!50}82.19 & \cellcolor{blue!50}5 & \cellcolor{blue!50}24.09 & \cellcolor{blue!50}83.18 & 82.14 & 6.84 & 52.72 & 2.88\\
 & 3 & 84 & 83.55 & \cellcolor{blue!50}7 & \cellcolor{blue!50}37.21 & \cellcolor{blue!50}85.07 & \cellcolor{blue!50}84.66 & 12.68 & 126.18 & 5.76\\
 & 4 & 84.58 & 83.84 & \cellcolor{blue!50}9.4 & \cellcolor{blue!50}52.06 & \cellcolor{blue!50}86.05 & \cellcolor{blue!50}84.78 & 18.68 & 315.45 & 8.64\\
 & 5 & 85.33 & 83.92 & \cellcolor{blue!50}11.72 & \cellcolor{blue!50}71.08 & \cellcolor{blue!50}86.44 & \cellcolor{blue!50}84.88 & 23.88 & 865.26 & 11.08\\
 & 6 & 86.26 & 83.70 & \cellcolor{blue!50}14.68 & \cellcolor{blue!50}99.47 & \cellcolor{blue!50}87.6 & \cellcolor{blue!50}85.76 & 39.16 & 2666.67 & 17.32\\
\hline

\multirow{5}{*}{\shortstack[1]{audiology}}
& 2 & 94.91 & \cellcolor{blue!50}94.92 & \cellcolor{blue!50}4 & \cellcolor{blue!50}10.59 & \cellcolor{blue!50}95.49 & \cellcolor{blue!50}94.92 & 7 & 31.35 & 3\\
 & 3 & 96.78 & \cellcolor{blue!50}95.84 & \cellcolor{blue!50}5.04 & \cellcolor{blue!50}16.41 & \cellcolor{blue!50}97.82 & 95.56 & 11.56 & 88.75 & 5.28\\
 & 4 & 97.73 & \cellcolor{blue!50}95.56 & \cellcolor{blue!50}6.96 & \cellcolor{blue!50}22.56 & \cellcolor{blue!50}99.51 & 94.54 & 19.08 & 272.15 & 8.68\\
 & 5 & 98.40 & \cellcolor{blue!50}94.44 & \cellcolor{blue!50}9.88 & \cellcolor{blue!50}29.82 & \cellcolor{blue!50}99.95 & 93.98 & 27 & 915.29 & 11.72\\
 & 6 & 99.17 & \cellcolor{blue!50}95.84 & \cellcolor{blue!50}14.28 & \cellcolor{blue!50}39.59 & \cellcolor{blue!50}99.86 & 94.08 & 24.12 & 3323.61 & 10.88\\
\hline

\multirow{5}{*}{\shortstack[1]{australian}}
 & 2 & 86.70 & \cellcolor{blue!50}85.94 & \cellcolor{blue!50}4.72 & \cellcolor{blue!50}26.79 & \cellcolor{blue!50}86.93 & 85.33 & 6.68 & 59.65 & 2.84\\
 & 3 & 87.45 & 84.81 & \cellcolor{blue!50}5.32 & \cellcolor{blue!50}41.15 & \cellcolor{blue!50}88.09 & \cellcolor{blue!50}84.87 & 13.08 & 146.15 & 5.68\\
 & 4 & 88.45 & \cellcolor{blue!50}86.03 & \cellcolor{blue!50}7.4 & \cellcolor{blue!50}56.85 & \cellcolor{blue!50}88.74 & 85.18 & 17.48 & 377.62 & 7.92\\
 & 5 & \cellcolor{blue!50}89.36 & \cellcolor{blue!50}85.91 & \cellcolor{blue!50}10.44 & \cellcolor{blue!50}75.9 & 89.28 & 84.75 & 22.52 & 1076.35 & 10.08\\
 & 6 & \cellcolor{blue!50}90.05 & \cellcolor{blue!50}85.7 & \cellcolor{blue!50}17.32 & \cellcolor{blue!50}102.49 & 89.49 & 84.84 & 27.08 & 3433.64 & 12.20\\
\hline

\multirow{5}{*}{\shortstack[1]{cancer}}
 & 2 & 93.88 & 93.59 & \cellcolor{blue!50}4 & \cellcolor{blue!50}20.29 & \cellcolor{blue!50}94.91 & \cellcolor{blue!50}94.2 & 7 & 45.56 & 3\\
 & 3 & 95.02 & 93.91 & \cellcolor{blue!50}5.84 & \cellcolor{blue!50}31.37 & \cellcolor{blue!50}96.6 & \cellcolor{blue!50}94.73 & 15 & 110.85 & 6.96\\
 & 4 & 96.06 & \cellcolor{blue!50}95.49 & \cellcolor{blue!50}7.96 & \cellcolor{blue!50}43.89 & \cellcolor{blue!50}97.34 & 94.17 & 21 & 283.77 & 9.44\\
 & 5 & 95.94 & 93.74 & \cellcolor{blue!50}10.68 & \cellcolor{blue!50}59.91 & \cellcolor{blue!50}97.99 & \cellcolor{blue!50}94.35 & 29.32 & 800.89 & 13.20\\
 & 6 & 96.84 & \cellcolor{blue!50}94.35 & \cellcolor{blue!50}14.8 & \cellcolor{blue!50}83.83 & \cellcolor{blue!50}98.87 & 93.41 & 45.72 & 2536.91 & 19.88\\
\hline

\multirow{5}{*}{\shortstack[1]{car}}
 & 2 & \cellcolor{blue!50}85.53 & \cellcolor{blue!50}85.53 & \cellcolor{blue!50}4 & \cellcolor{blue!50}13.32 & \cellcolor{blue!50}85.53 & \cellcolor{blue!50}85.53 & 6.84 & 32.01 & 2.92\\
 & 3 & 88.40 & 87.41 & \cellcolor{blue!50}5.08 & \cellcolor{blue!50}21.95 & \cellcolor{blue!50}89.25 & \cellcolor{blue!50}87.45 & 12.68 & 71.83 & 5.64\\
 & 4 & 89.84 & 88.54 & \cellcolor{blue!50}6.84 & \cellcolor{blue!50}34.44 & \cellcolor{blue!50}91.62 & \cellcolor{blue!50}89.68 & 20.36 & 162.46 & 7.68\\
 & 5 & 91.13 & 89.91 & \cellcolor{blue!50}9.6 & \cellcolor{blue!50}55.79 & \cellcolor{blue!50}93.78 & \cellcolor{blue!50}92.77 & 29.56 & 389.68 & 10.24\\
 & 6 & 93.51 & 92.99 & \cellcolor{blue!50}13.36 & \cellcolor{blue!50}97.06 & \cellcolor{blue!50}95.8 & \cellcolor{blue!50}95.06 & 31.96 & 1044.54 & 10.88\\
\hline

\multirow{5}{*}{\shortstack[1]{cleveland}}
  & 2 & 79.04 & 72.57 & \cellcolor{blue!50}4 & \cellcolor{blue!50}9.48 & \cellcolor{blue!50}80.76 & \cellcolor{blue!50}72.84 & 7 & 25.57 & 3\\
 & 3 & 85.07 & \cellcolor{blue!50}83.37 & \cellcolor{blue!50}6 & \cellcolor{blue!50}14.73 & \cellcolor{blue!50}85.68 & 76.55 & 12.84 & 68.93 & 5.72\\
 & 4 & 86.32 & \cellcolor{blue!50}79.46 & \cellcolor{blue!50}7.84 & \cellcolor{blue!50}20.55 & \cellcolor{blue!50}86.77 & 76.75 & 17.80 & 200.76 & 8.04\\
 & 5 & \cellcolor{blue!50}88.65 & \cellcolor{blue!50}78.72 & \cellcolor{blue!50}13.08 & \cellcolor{blue!50}27.89 & 87.26 & 74.45 & 23.96 & 646.75 & 10.84\\
 & 6 & \cellcolor{blue!50}90.74 & \cellcolor{blue!50}77.29 & \cellcolor{blue!50}21.04 & \cellcolor{blue!50}38.66 & 88.58 & 75.81 & 28.84 & 2284.76 & 13.08\\
\hline

\multirow{5}{*}{\shortstack[1]{hypothyroid}}
  & 2 & \cellcolor{blue!50}97.84 & \cellcolor{blue!50}97.84 & \cellcolor{blue!50}4 & \cellcolor{blue!50}92.65 & \cellcolor{blue!50}97.84 & \cellcolor{blue!50}97.84 & 5.96 & 182.20 & 2.48\\
 & 3 & 98.09 & \cellcolor{blue!50}98.04 & \cellcolor{blue!50}5.12 & \cellcolor{blue!50}142.78 & \cellcolor{blue!50}98.14 & 97.82 & 9.72 & 402.98 & 4.32\\
 & 4 & 98.27 & \cellcolor{blue!50}98.13 & \cellcolor{blue!50}6.72 & \cellcolor{blue!50}200.09 & \cellcolor{blue!50}98.38 & 98.01 & 15.40 & 885.51 & 7.12\\
 & 5 & 98.30 & \cellcolor{blue!50}98.05 & \cellcolor{blue!50}9.28 & \cellcolor{blue!50}274.03 & \cellcolor{blue!50}98.45 & 98 & 20.04 & 2016.31 & 8.92\\
 & 6 & 98.37 & \cellcolor{blue!50}97.95 & \cellcolor{blue!50}13.68 & \cellcolor{blue!50}385.4 & \cellcolor{blue!50}98.46 & 97.91 & 33.16 & 4957.57 & 14.04\\
\hline

\multirow{5}{*}{\shortstack[1]{kr-vs-kp}}
 & 2 & 77.83 & 77.01 & \cellcolor{blue!50}4 & \cellcolor{blue!50}77.88 & \cellcolor{blue!50}86.92 & \cellcolor{blue!50}86.92 & 7 & 155.09 & 3\\
 & 3 & 90.43 & 90.43 & \cellcolor{blue!50}5.28 & \cellcolor{blue!50}120.54 & \cellcolor{blue!50}93.81 & \cellcolor{blue!50}93.79 & 12.44 & 342.99 & 5.08\\
 & 4 & 94.09 & 94.09 & \cellcolor{blue!50}7.56 & \cellcolor{blue!50}170.28 & \cellcolor{blue!50}94.32 & \cellcolor{blue!50}94.14 & 17.24 & 753.78 & 7.12\\
 & 5 & 94.34 & 94.18 & \cellcolor{blue!50}9.52 & \cellcolor{blue!50}236.39 & \cellcolor{blue!50}94.85 & \cellcolor{blue!50}94.69 & 25.40 & 1717.14 & 10.20\\
 & 6 & 92.80 & 92.55 & \cellcolor{blue!50}11.52 & \cellcolor{blue!50}339.35 & \cellcolor{blue!50}93.91 & \cellcolor{blue!50}93.69 & 29.32 & 4227.67 & 12.20\\
\hline

\multirow{5}{*}{\shortstack[1]{lymph}}
& 2 & 84.46 & \cellcolor{blue!50}83.23 & \cellcolor{blue!50}4 & \cellcolor{blue!50}3.5 & \cellcolor{blue!50}86.01 & 79.27 & 7 & 12.33 & 3\\
 & 3 & 86.76 & 78.35 & \cellcolor{blue!50}5.92 & \cellcolor{blue!50}5.55 & \cellcolor{blue!50}91.93 & \cellcolor{blue!50}80.54 & 14.68 & 36.65 & 6.64\\
 & 4 & 90.54 & \cellcolor{blue!50}82.4 & \cellcolor{blue!50}8.72 & \cellcolor{blue!50}7.86 & \cellcolor{blue!50}94.56 & 78.46 & 20.20 & 117.94 & 8.88\\
 & 5 & 93.51 & \cellcolor{blue!50}83.6 & \cellcolor{blue!50}13.52 & \cellcolor{blue!50}10.94 & \cellcolor{blue!50}97.09 & 82.46 & 27.08 & 413.09 & 11.88\\
 & 6 & 95.88 & \cellcolor{blue!50}84.82 & \cellcolor{blue!50}17.64 & \cellcolor{blue!50}15.74 & \cellcolor{blue!50}99.59 & 80.92 & 46.60 & 1550.34 & 18.96\\
\hline

\multirow{5}{*}{\shortstack[1]{mushroom}}
& 2 & 95.13 & 95.13 & \cellcolor{blue!50}4 & \cellcolor{blue!50}299.19 & \cellcolor{blue!50}96.9 & \cellcolor{blue!50}96.9 & 7 & 565.27 & 3\\
 & 3 & 97.74 & 97.77 & \cellcolor{blue!50}6.8 & \cellcolor{blue!50}458.1 & \cellcolor{blue!50}99.9 & \cellcolor{blue!50}99.9 & 13.72 & 1227.18 & 6.24\\
 & 4 & 98.78 & 98.74 & \cellcolor{blue!50}9 & \cellcolor{blue!50}635.09 & \cellcolor{blue!50}100 & \cellcolor{blue!50}100 & 19.80 & 2603.94 & 9.08\\
 & 5 & 98.63 & 98.57 & \cellcolor{blue!50}11.32 & \cellcolor{blue!50}853.68 & \cellcolor{blue!50}100 & \cellcolor{blue!50}100 & 23.40 & 5571.14 & 10.64\\
 & 6 & 97.28 & 97.10 & \cellcolor{blue!50}14.6 & \cellcolor{blue!50}1165.88 & \cellcolor{blue!50}100 & \cellcolor{blue!50}100 & 27.56 & 12376.90 & 12\\
\hline

\multirow{5}{*}{\shortstack[1]{tumor}}
& 2 & 82.80 & \cellcolor{blue!50}81.6 & \cellcolor{blue!50}4 & \cellcolor{blue!50}3.72 & \cellcolor{blue!50}82.92 & 81.01 & 6.76 & 10.46 & 2.88\\
 & 3 & 83.84 & 80.43 & \cellcolor{blue!50}5.3 & \cellcolor{blue!50}6.02 & \cellcolor{blue!50}86.16 & \cellcolor{blue!50}82.97 & 13.88 & 27.24 & 6.08\\
 & 4 & 85.52 & 82.49 & \cellcolor{blue!50}8.64 & \cellcolor{blue!50}9.04 & \cellcolor{blue!50}87.89 & \cellcolor{blue!50}82.85 & 20.92 & 76.40 & 9.16\\
 & 5 & 87.51 & \cellcolor{blue!50}85.83 & \cellcolor{blue!50}13.32 & \cellcolor{blue!50}13.79 & \cellcolor{blue!50}90.1 & 79.34 & 47.80 & 239.09 & 16.84\\
 & 6 & 88.57 & 81.12 & \cellcolor{blue!50}19.84 & \cellcolor{blue!50}22.44 & \cellcolor{blue!50}90.34 & \cellcolor{blue!50}81.31 & 37.32 & 838.63 & 15.04\\
\hline

\multirow{5}{*}{\shortstack[1]{soybean}}
  & 2 & 90.48 & 90.48 & \cellcolor{blue!50}4 & \cellcolor{blue!50}10.79 & \cellcolor{blue!50}91.27 & \cellcolor{blue!50}91.27 & 7 & 25.55 & 3\\
 & 3 & 91.39 & 90.41 & \cellcolor{blue!50}6.52 & \cellcolor{blue!50}16.99 & \cellcolor{blue!50}95.45 & \cellcolor{blue!50}94.7 & 15 & 62.30 & 7\\
 & 4 & 93.24 & 93.21 & \cellcolor{blue!50}9.04 & \cellcolor{blue!50}24.54 & \cellcolor{blue!50}97.25 & \cellcolor{blue!50}95.9 & 22.20 & 160.18 & 9.88\\
 & 5 & 94.31 & 92.95 & \cellcolor{blue!50}11.92 & \cellcolor{blue!50}35.34 & \cellcolor{blue!50}97.96 & \cellcolor{blue!50}95.3 & 40.60 & 455.33 & 15.72\\
 & 6 & 96.07 & 95.52 & \cellcolor{blue!50}14.88 & \cellcolor{blue!50}53.41 & \cellcolor{blue!50}98.27 & \cellcolor{blue!50}96.03 & 33.40 & 1459.87 & 14.40\\
\hline

\multirow{5}{*}{\shortstack[1]{splice-1}}
& 2 & 84.04 & \cellcolor{blue!50}84.04 & \cellcolor{blue!50}4 & \cellcolor{blue!50}296.61 & \cellcolor{blue!50}84.22 & 83.17 & 6.92 & 555.22 & 2.96\\
 & 3 & 87.25 & 86.94 & \cellcolor{blue!50}5.44 & \cellcolor{blue!50}449.04 & \cellcolor{blue!50}87.79 & \cellcolor{blue!50}87.37 & 11.32 & 1231.59 & 4.64\\
 & 4 & \cellcolor{blue!50}88.3 & \cellcolor{blue!50}88.04 & \cellcolor{blue!50}7.24 & \cellcolor{blue!50}608.3 & 86.52 & 85.64 & 16.60 & 2717.90 & 7.12\\
 & 5 & 71.99 & 70.53 & \cellcolor{blue!50}10.28 & \cellcolor{blue!50}783.9 & \cellcolor{blue!50}77.37 & \cellcolor{blue!50}76.32 & 21.88 & 6226.75 & 9.48\\
 & 6 & \cellcolor{blue!50}62.92 & \cellcolor{blue!50}61.89 & \cellcolor{blue!50}16.28 & \cellcolor{blue!50}996.27 & 60.36 & 58.95 & 29.40 & 15406.05 & 12.28\\
\hline

\multirow{5}{*}{\shortstack[1]{tic-tac-toe}}
  & 2 & 71.05 & \cellcolor{blue!50}68.35 & \cellcolor{blue!50}4 & \cellcolor{blue!50}9.25 & \cellcolor{blue!50}71.1 & 67.49 & 5.96 & 22.31 & 2.48\\
 & 3 & 74.91 & 72.36 & \cellcolor{blue!50}6.16 & \cellcolor{blue!50}15.01 & \cellcolor{blue!50}77.15 & \cellcolor{blue!50}73.55 & 11.48 & 51.98 & 5.20\\
 & 4 & 76.87 & 74.22 & \cellcolor{blue!50}8.84 & \cellcolor{blue!50}22.88 & \cellcolor{blue!50}82.47 & \cellcolor{blue!50}78.68 & 20.60 & 125.10 & 8.44\\
 & 5 & 81.86 & \cellcolor{blue!50}80.31 & \cellcolor{blue!50}13.88 & \cellcolor{blue!50}35.67 & \cellcolor{blue!50}83.08 & 79.50 & 28.44 & 328.33 & 11.16\\
 & 6 & \cellcolor{blue!50}84.82 & 80.08 & \cellcolor{blue!50}24.16 & \cellcolor{blue!50}59.52 & 84.25 & \cellcolor{blue!50}80.86 & 38.12 & 979.46 & 13.24\\
\hline

\multirow{5}{*}{\shortstack[1]{vote}}
 & 2 & 95.68 & \cellcolor{blue!50}95.22 & \cellcolor{blue!50}3.76 & \cellcolor{blue!50}7.2 & \cellcolor{blue!50}96.21 & 95.03 & 7 & 18.33 & 3\\
 & 3 & 96.69 & \cellcolor{blue!50}94.57 & \cellcolor{blue!50}5.56 & \cellcolor{blue!50}11.38 & \cellcolor{blue!50}97.39 & 93.79 & 13.96 & 46.55 & 6.04\\
 & 4 & 97.40 & 94.39 & \cellcolor{blue!50}8.16 & \cellcolor{blue!50}16.49 & \cellcolor{blue!50}98.62 & \cellcolor{blue!50}94.57 & 21.16 & 126.45 & 9.32\\
 & 5 & 98.21 & \cellcolor{blue!50}94.57 & \cellcolor{blue!50}12.4 & \cellcolor{blue!50}23.83 & \cellcolor{blue!50}99.47 & 93.84 & 30.52 & 381.95 & 12.96\\
 & 6 & 98.93 & 93.98 & \cellcolor{blue!50}18.44 & \cellcolor{blue!50}36.2 & \cellcolor{blue!50}99.62 & \cellcolor{blue!50}94.76 & 35.40 & 1292.44 & 14.88\\
\hline

    \end{tabular}
    \caption{\label{tab:exp3_maxsatbdd_maxsatdt_all}
    \footnotesize{Comparison between \MaxSat-\Bdd{} and \MaxSat-\DT{}.}
    }
\end{table}


\subsection{Evaluation of a Heuristic \MaxSat-\Bdd{} method}

To increase the scalability of our model, we propose a simple heuristic version of \MaxSat-\Bdd{}. 
The idea is to perform as a pre-processing step to choose a subset of (important) features that are used exclusively in the \MaxSat-\Bdd{} model.
By doing this, the search space is greatly reduced by focusing only on the selected features. 
We chose to run $\cart{}$~\cite{DBLP:books/wa/BreimanFOS84} (an efficient and scalable heuristic for learning decision trees) to build quickly a decision tree. 
The features selected in our heuristic method are the ones used in the decision tree found by $\cart{}$. 

{This experimental study follows the same protocol as the previous one.} 
{The results are detailed in Figure~\ref{fig:exp4_combined} and in Table~\ref{tab:annex2_cart_hmaxsatbdd_maxsatbdd}}.


As expected, the size of the encoding of the model is in favor of the heuristic approach. This is shown in the two columns \textbf{E\_size} of Table~\ref{tab:annex2_cart_hmaxsatbdd_maxsatbdd}.
This advantage gives our approach the key to handle larger problems.
Column \textbf{Opt} of Table~\ref{tab:annex2_cart_hmaxsatbdd_maxsatbdd} indicates the percentage of instances solved to optimality. 
It should be noted that proving optimality is much easier with the heuristic approach since the problem is naturally easier to solve with fewer features.


The results of the average testing accuracy are shown in the left scatter of Figure~\ref{fig:exp4_combined}.
Our heuristic approach is clearly very competitive to the exact \MaxSat-\Bdd{} in terms of learning generalization.
This is particularly true for datasets with a large number of features. 
Indeed, the proposed heuristic approach obtains better prediction performance than the exact one within the same limited resources (time and memory). 
The middle and right scatters of Figure \ref{fig:exp4_combined} show the comparison of the average training and testing accuracy between the heuristic approach and \cart{}.
It is clear that \cart{} almost always gets better training accuracy. 
However, the heuristic \MaxSat-\Bdd{} is still competitive in terms of generalisation.


\begin{figure*}[ht!]
    \centering
    \begin{minipage}{0.31\linewidth}
        \centering
        \includegraphics[width=.8\linewidth, height=.7\linewidth]
        {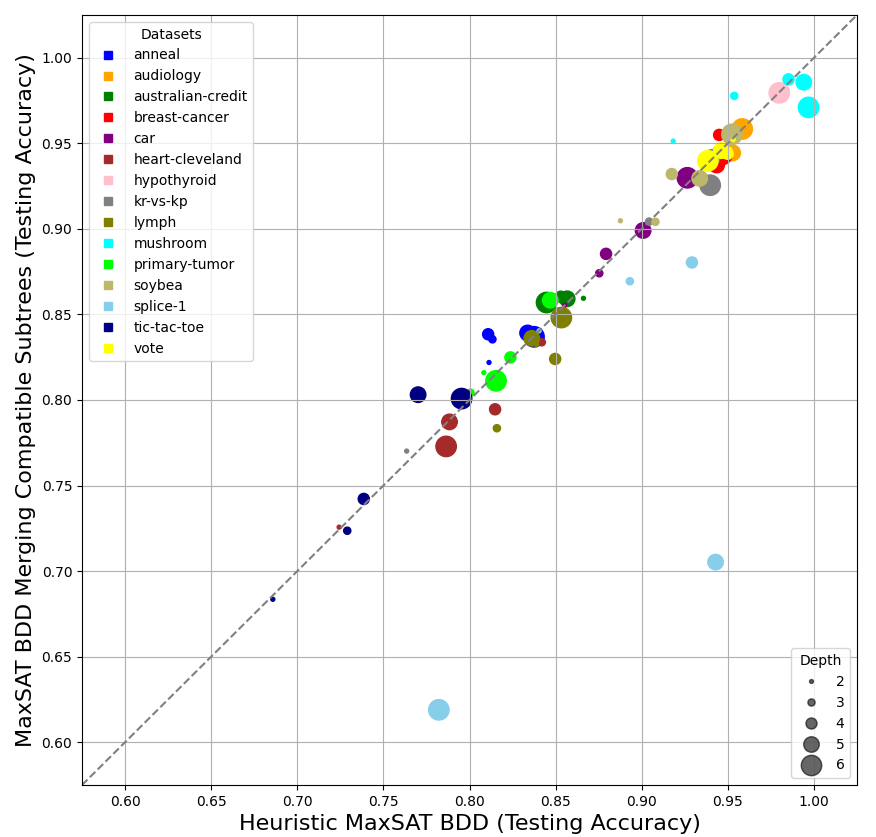}
    \end{minipage}\hfill
    \begin{minipage}{0.31\linewidth}
        \centering
        \includegraphics[width=.8\linewidth, height=.7\linewidth]
        {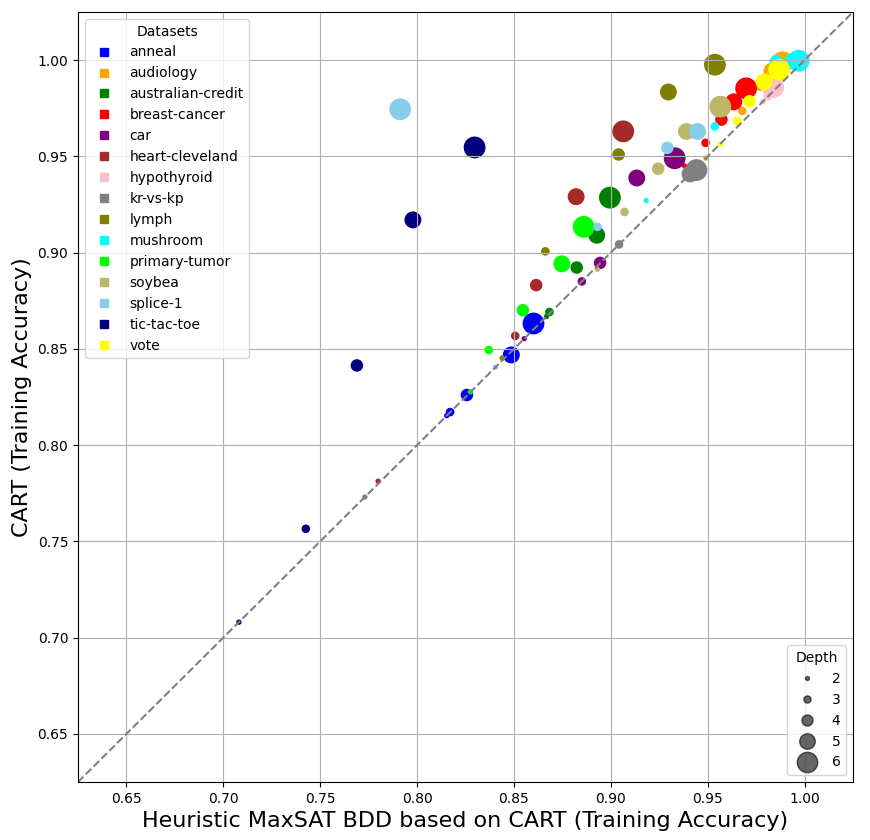}
    \end{minipage}\hfill
    \begin{minipage}{0.31\linewidth}
        \centering
        \includegraphics[width=.8\linewidth, height=.7\linewidth]
        {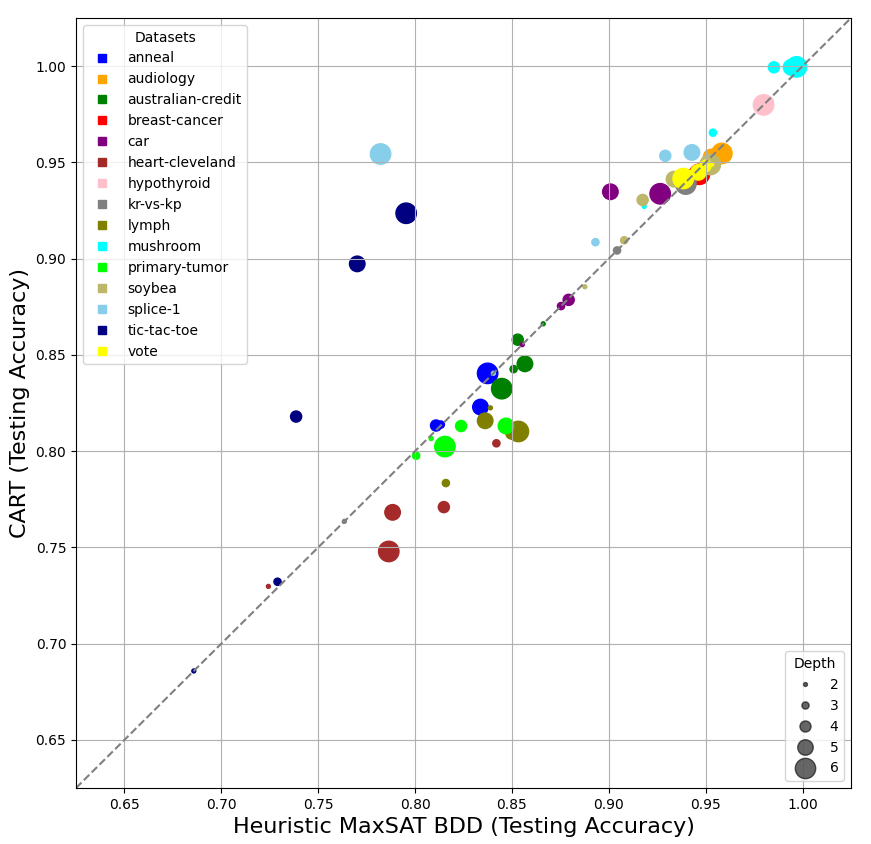}
    \end{minipage}
    \caption{{Left: comparison of Heuristic \MaxSat-\Bdd{} and Excact \MaxSat-\Bdd{} in testing accuracy. Middle: comparison of the training accuracy of Heuristic \MaxSat-\Bdd{} and \cart{}. Right: comparison of the training accuracy of Heuristic \MaxSat-\Bdd{} and \cart{}}}
    \label{fig:exp4_combined}
\end{figure*}

\begin{table*}[htb!]
    \centering
    \renewcommand{\arraystretch}{0.4}
    \setlength{\tabcolsep}{5pt}
    \scriptsize
    \begin{tabular}{|c|c||c|c|c|c||c|c|c|c|c|c||c|c|c|c|c|c||}
    \hline
    \multirow{2}{*}{\textbf{Datasets}} & 
    \multirow{2}{*}{{\textbf{d}}} &
    \multicolumn{4}{c||}{{\textbf{CART}}} & 
    \multicolumn{6}{c||}{{\textbf{Heuristic \MaxSat-\Bdd{}}}} &
    \multicolumn{6}{c||}{{\textbf{\MaxSat-\Bdd{}}}} \\
    \cline{3-18}
    & & 
    \textbf{Train} &
    \textbf{Test} &
    \textbf{Size} &
    \textbf{F\_d} &
    \textbf{Opt} &
    \textbf{Train} &
    \textbf{Test} &
    \textbf{Size} &
    \textbf{E\_size} &
    \textbf{Time} &
   \textbf{Opt} &
    \textbf{Train} &
    \textbf{Test} &
    \textbf{Size} &
    \textbf{E\_size} &
    \textbf{Time}
    \\
    \hline
    
    \multirow{5}{*}{anneal}
 & 2 & 81.53 & 81.21 & 6.12 & 2.56 & \cellcolor{blue!50}100 & 81.53 & 81.13 & \cellcolor{blue!50}3.56 & \cellcolor{blue!50}1.45 & \cellcolor{blue!50}0.13 & \cellcolor{blue!50}100 & \cellcolor{blue!50}82.92 & \cellcolor{blue!50}82.19 & 5 & 24.09 & 92.93\\
 & 3 & 81.72 & 81.38 & 11.08 & 4.92 & \cellcolor{blue!50}100 & 81.71 & 81.33 & \cellcolor{blue!50}5.24 & \cellcolor{blue!50}4.03 & \cellcolor{blue!50}1.64 & 0 & \cellcolor{blue!50}84 & \cellcolor{blue!50}83.55 & 7 & 37.21 & TO \\
 & 4 & 82.60 & 81.33 & 18.04 & 8.40 & \cellcolor{blue!50}100 & 82.57 & 81.08 & \cellcolor{blue!50}7.08 & \cellcolor{blue!50}9.64 & \cellcolor{blue!50}109.65 & 0 & \cellcolor{blue!50}84.58 & \cellcolor{blue!50}83.84 & 9.40 & 52.06 & TO \\
 & 5 & 84.69 & 82.29 & 27.88 & 12.32 & \cellcolor{blue!50}12 & 84.86 & 83.37 & \cellcolor{blue!50}11.12 & \cellcolor{blue!50}20.62 & \cellcolor{blue!50}780.08 & 0 & \cellcolor{blue!50}85.33 & \cellcolor{blue!50}83.92 & 11.72 & 71.08 & TO \\
 & 6 & \cellcolor{blue!50}86.32 & \cellcolor{blue!50}84.04 & 39.80 & 17 & 0 & 86.01 & 83.74 & \cellcolor{blue!50}13.56 & \cellcolor{blue!50}42.6 & \cellcolor{blue!50}845.72 & 0 & 86.26 & 83.70 & 14.68 & 99.47 & TO \\
\hline

\multirow{5}{*}{audiology}
 & 2 & \cellcolor{blue!50}94.91 & \cellcolor{blue!50}94.92 & 5 & 2 & \cellcolor{blue!50}100 & \cellcolor{blue!50}94.91 & \cellcolor{blue!50}94.92 & \cellcolor{blue!50}4 & \cellcolor{blue!50}0.35 & \cellcolor{blue!50}0.01 & \cellcolor{blue!50}100 & \cellcolor{blue!50}94.91 & \cellcolor{blue!50}94.92 & \cellcolor{blue!50}4 & 10.59 & 0.46\\
 & 3 & \cellcolor{blue!50}97.36 & 94.82 & 9 & 4 & \cellcolor{blue!50}100 & 96.78 & 95.38 & \cellcolor{blue!50}5.04 & \cellcolor{blue!50}1 & \cellcolor{blue!50}0.02 & \cellcolor{blue!50}100 & 96.78 & \cellcolor{blue!50}95.84 & \cellcolor{blue!50}5.04 & 16.41 & 6.63\\
 & 4 & \cellcolor{blue!50}98.73 & 95.37 & 13.08 & 6 & \cellcolor{blue!50}100 & 97.73 & \cellcolor{blue!50}95.56 & 7.04 & \cellcolor{blue!50}2.27 & \cellcolor{blue!50}0.08 & \cellcolor{blue!50}100 & 97.73 & \cellcolor{blue!50}95.56 & \cellcolor{blue!50}6.96 & 22.56 & 56.31\\
 & 5 & \cellcolor{blue!50}99.42 & \cellcolor{blue!50}95.28 & 17.08 & 8 & \cellcolor{blue!50}100 & 98.31 & \cellcolor{blue!50}95.28 & \cellcolor{blue!50}9.76 & \cellcolor{blue!50}4.8 & \cellcolor{blue!50}0.49 & 72 & 98.40 & 94.44 & 9.88 & 29.82 & 578.99\\
 & 6 & \cellcolor{blue!50}99.88 & 95.47 & 19.08 & 9 & \cellcolor{blue!50}100 & 98.87 & \cellcolor{blue!50}95.84 & \cellcolor{blue!50}13 & \cellcolor{blue!50}9.78 & \cellcolor{blue!50}2.13 & 48 & 99.17 & \cellcolor{blue!50}95.84 & 14.28 & 39.59 & 613.06\\
\hline

\multirow{5}{*}{australian}
 & 2 & 86.68 & \cellcolor{blue!50}86.62 & 7 & 3 & \cellcolor{blue!50}100 & 86.68 & \cellcolor{blue!50}86.62 & 4.92 & \cellcolor{blue!50}1.26 & \cellcolor{blue!50}0.09 & \cellcolor{blue!50}100 & \cellcolor{blue!50}86.7 & 85.94 & \cellcolor{blue!50}4.72 & 26.79 & 167.99\\
 & 3 & 86.91 & 84.26 & 13.08 & 6 & \cellcolor{blue!50}100 & 86.83 & \cellcolor{blue!50}85.09 & 5.48 & \cellcolor{blue!50}3.59 & \cellcolor{blue!50}2.41 & 0 & \cellcolor{blue!50}87.45 & 84.81 & \cellcolor{blue!50}5.32 & 41.15 & TO \\
 & 4 & \cellcolor{blue!50}89.23 & 85.79 & 24.92 & 11.84 & \cellcolor{blue!50}84 & 88.24 & 85.30 & \cellcolor{blue!50}6.8 & \cellcolor{blue!50}9.22 & \cellcolor{blue!50}536.16 & 0 & 88.45 & \cellcolor{blue!50}86.03 & 7.40 & 56.85 & TO \\
 & 5 & \cellcolor{blue!50}90.9 & 84.53 & 41.64 & 19.24 & 0 & 89.27 & 85.67 & 10.64 & \cellcolor{blue!50}20.28 & \cellcolor{blue!50}845.28 & 0 & 89.36 & \cellcolor{blue!50}85.91 & \cellcolor{blue!50}10.44 & 75.90 & TO \\
 & 6 & \cellcolor{blue!50}92.86 & 83.24 & 64.28 & 28.84 & 0 & 89.95 & 84.47 & \cellcolor{blue!50}16.12 & \cellcolor{blue!50}41.85 & \cellcolor{blue!50}TO  & 0 & 90.05 & \cellcolor{blue!50}85.7 & 17.32 & 102.49 & \cellcolor{blue!50}TO \\
\hline

\multirow{5}{*}{cancer}
 & 2 & \cellcolor{blue!50}94.5 & \cellcolor{blue!50}93.91 & 7 & 3 & \cellcolor{blue!50}100 & 93.81 & 93.59 & \cellcolor{blue!50}4 & \cellcolor{blue!50}1.32 & \cellcolor{blue!50}0.06 & \cellcolor{blue!50}100 & 93.88 & 93.59 & \cellcolor{blue!50}4 & 20.29 & 5.89\\
 & 3 & \cellcolor{blue!50}95.7 & \cellcolor{blue!50}94.41 & 13.24 & 6.08 & \cellcolor{blue!50}100 & 94.89 & 94.14 & \cellcolor{blue!50}5.64 & \cellcolor{blue!50}3.78 & \cellcolor{blue!50}0.52 & \cellcolor{blue!50}100 & 95.02 & 93.91 & 5.84 & 31.37 & 525.59\\
 & 4 & \cellcolor{blue!50}96.91 & 94.26 & 21.08 & 9.88 & \cellcolor{blue!50}100 & 95.71 & 94.50 & \cellcolor{blue!50}7.8 & \cellcolor{blue!50}8.77 & \cellcolor{blue!50}20.44 & 0 & 96.06 & \cellcolor{blue!50}95.49 & 7.96 & 43.89 & TO \\
 & 5 & \cellcolor{blue!50}97.83 & 94.20 & 30.36 & 14.04 & \cellcolor{blue!50}60 & 96.35 & \cellcolor{blue!50}94.35 & 10.92 & \cellcolor{blue!50}18.32 & \cellcolor{blue!50}637.98 & 0 & 95.94 & 93.74 & \cellcolor{blue!50}10.68 & 59.91 & TO \\
 & 6 & \cellcolor{blue!50}98.54 & 94.38 & 38.84 & 17.68 & 0 & 96.98 & \cellcolor{blue!50}94.67 & 15.20 & \cellcolor{blue!50}36.33 & \cellcolor{blue!50}864.36 & 0 & 96.84 & 94.35 & \cellcolor{blue!50}14.8 & 83.83 & TO \\
\hline

\multirow{5}{*}{car}
 & 2 & \cellcolor{blue!50}85.53 & \cellcolor{blue!50}85.53 & 5 & 2 & \cellcolor{blue!50}100 & \cellcolor{blue!50}85.53 & \cellcolor{blue!50}85.53 & \cellcolor{blue!50}4 & \cellcolor{blue!50}2.77 & \cellcolor{blue!50}0.14 & \cellcolor{blue!50}100 & \cellcolor{blue!50}85.53 & \cellcolor{blue!50}85.53 & \cellcolor{blue!50}4 & 13.32 & 24.82\\
 & 3 & \cellcolor{blue!50}88.5 & \cellcolor{blue!50}87.53 & 7 & 3 & \cellcolor{blue!50}100 & \cellcolor{blue!50}88.5 & \cellcolor{blue!50}87.53 & \cellcolor{blue!50}5 & \cellcolor{blue!50}6.94 & \cellcolor{blue!50}0.74 & 8 & 88.40 & 87.41 & 5.08 & 21.95 & TO \\
 & 4 & 89.46 & 87.86 & 11 & 5 & \cellcolor{blue!50}100 & 89.45 & 87.93 & \cellcolor{blue!50}6.4 & \cellcolor{blue!50}16.64 & \cellcolor{blue!50}14.83 & 0 & \cellcolor{blue!50}89.84 & \cellcolor{blue!50}88.54 & 6.84 & 34.44 & TO \\
 & 5 & \cellcolor{blue!50}93.88 & \cellcolor{blue!50}93.47 & 18.20 & 7.80 & \cellcolor{blue!50}24 & 91.34 & 90.08 & \cellcolor{blue!50}9.24 & \cellcolor{blue!50}37.43 & \cellcolor{blue!50}843.14 & 0 & 91.13 & 89.91 & 9.60 & 55.79 & TO \\
 & 6 & \cellcolor{blue!50}94.9 & \cellcolor{blue!50}93.37 & 28.68 & 10.32 & 0 & 93.30 & 92.65 & \cellcolor{blue!50}11.76 & \cellcolor{blue!50}79.23 & \cellcolor{blue!50}TO  & 0 & 93.51 & 92.99 & 13.36 & 97.06 & \cellcolor{blue!50}TO \\
\hline

\multirow{5}{*}{cleveland}
 & 2 & 78.13 & \cellcolor{blue!50}72.97 & 7 & 2.72 & \cellcolor{blue!50}100 & 77.99 & 72.43 & \cellcolor{blue!50}3.76 & \cellcolor{blue!50}0.55 & \cellcolor{blue!50}0.04 & \cellcolor{blue!50}100 & \cellcolor{blue!50}79.04 & 72.57 & 4 & 9.48 & 83.84\\
 & 3 & \cellcolor{blue!50}85.68 & 80.41 & 15 & 6.24 & \cellcolor{blue!50}100 & 85.07 & \cellcolor{blue!50}84.2 & \cellcolor{blue!50}6 & \cellcolor{blue!50}1.68 & \cellcolor{blue!50}2.28 & 0 & 85.07 & 83.37 & \cellcolor{blue!50}6 & 14.73 & TO \\
 & 4 & \cellcolor{blue!50}88.31 & 77.09 & 29.96 & 13 & \cellcolor{blue!50}24 & 86.15 & \cellcolor{blue!50}81.49 & \cellcolor{blue!50}7.6 & \cellcolor{blue!50}4.46 & \cellcolor{blue!50}811.89 & 0 & 86.32 & 79.46 & 7.84 & 20.55 & TO \\
 & 5 & \cellcolor{blue!50}92.9 & 76.82 & 49.88 & 21.36 & 0 & 88.21 & \cellcolor{blue!50}78.84 & 13.24 & \cellcolor{blue!50}9.82 & \cellcolor{blue!50}862.94 & 0 & 88.65 & 78.72 & \cellcolor{blue!50}13.08 & 27.89 & TO \\
 & 6 & \cellcolor{blue!50}96.3 & 74.79 & 67.80 & 28.92 & 0 & 90.64 & \cellcolor{blue!50}78.64 & \cellcolor{blue!50}20.2 & \cellcolor{blue!50}19.2 & \cellcolor{blue!50}TO  & 0 & 90.74 & 77.29 & 21.04 & 38.66 & \cellcolor{blue!50}TO \\
\hline

\multirow{5}{*}{hypothyroid}
 & 2 & \cellcolor{blue!50}97.84 & \cellcolor{blue!50}97.84 & 6.92 & 2.96 & \cellcolor{blue!50}100 & \cellcolor{blue!50}97.84 & \cellcolor{blue!50}97.84 & \cellcolor{blue!50}4 & \cellcolor{blue!50}6.2 & \cellcolor{blue!50}0.43 & \cellcolor{blue!50}100 & \cellcolor{blue!50}97.84 & \cellcolor{blue!50}97.84 & \cellcolor{blue!50}4 & 92.65 & 77.76\\
 & 3 & \cellcolor{blue!50}98.13 & 97.86 & 12.84 & 5.52 & \cellcolor{blue!50}100 & 98.09 & 97.99 & 5.16 & \cellcolor{blue!50}16.95 & \cellcolor{blue!50}4.62 & 0 & 98.09 & \cellcolor{blue!50}98.04 & \cellcolor{blue!50}5.12 & 142.78 & TO \\
 & 4 & \cellcolor{blue!50}98.39 & 98.15 & 22.04 & 9.80 & \cellcolor{blue!50}100 & 98.28 & \cellcolor{blue!50}98.2 & \cellcolor{blue!50}6.56 & \cellcolor{blue!50}41.23 & \cellcolor{blue!50}262.45 & 0 & 98.27 & 98.13 & 6.72 & 200.09 & TO \\
 & 5 & \cellcolor{blue!50}98.48 & 98.04 & 31.72 & 14.24 & 0 & 98.32 & \cellcolor{blue!50}98.07 & \cellcolor{blue!50}8.84 & \cellcolor{blue!50}87.02 & \cellcolor{blue!50}TO  & 0 & 98.30 & 98.05 & 9.28 & 274.03 & \cellcolor{blue!50}TO \\
 & 6 & \cellcolor{blue!50}98.6 & \cellcolor{blue!50}97.99 & 43.56 & 18.92 & 0 & 98.37 & \cellcolor{blue!50}97.99 & \cellcolor{blue!50}13.32 & \cellcolor{blue!50}175.62 & \cellcolor{blue!50}TO  & 0 & 98.37 & 97.95 & 13.68 & 385.40 & \cellcolor{blue!50}TO \\
\hline

\multirow{5}{*}{kr-vs-kp}
 & 2 & 77.30 & 76.35 & 5 & 2 & \cellcolor{blue!50}100 & 77.30 & 76.35 & \cellcolor{blue!50}4 & \cellcolor{blue!50}5.12 & \cellcolor{blue!50}0.55 & 0 & \cellcolor{blue!50}77.83 & \cellcolor{blue!50}77.01 & \cellcolor{blue!50}4 & 77.88 & TO \\
 & 3 & \cellcolor{blue!50}90.43 & \cellcolor{blue!50}90.43 & 8.44 & 3.72 & \cellcolor{blue!50}100 & \cellcolor{blue!50}90.43 & \cellcolor{blue!50}90.43 & 5.40 & \cellcolor{blue!50}13.92 & \cellcolor{blue!50}5.77 & 0 & \cellcolor{blue!50}90.43 & \cellcolor{blue!50}90.43 & \cellcolor{blue!50}5.28 & 120.54 & TO \\
 & 4 & \cellcolor{blue!50}94.09 & \cellcolor{blue!50}94.09 & 13.88 & 6.44 & \cellcolor{blue!50}100 & \cellcolor{blue!50}94.09 & \cellcolor{blue!50}94.09 & 7.68 & \cellcolor{blue!50}33.69 & \cellcolor{blue!50}56.23 & 0 & \cellcolor{blue!50}94.09 & \cellcolor{blue!50}94.09 & \cellcolor{blue!50}7.56 & 170.28 & TO \\
 & 5 & 94.09 & 94.09 & 21.88 & 9.96 & \cellcolor{blue!50}20 & 94.09 & 94.09 & \cellcolor{blue!50}8.48 & \cellcolor{blue!50}74.68 & \cellcolor{blue!50}795.33 & 0 & \cellcolor{blue!50}94.34 & \cellcolor{blue!50}94.18 & 9.52 & 236.39 & TO \\
 & 6 & 94.29 & 93.87 & 31.32 & 14.04 & 0 & \cellcolor{blue!50}94.42 & \cellcolor{blue!50}93.97 & 11.80 & \cellcolor{blue!50}157.84 & \cellcolor{blue!50}846.3 & 0 & 92.80 & 92.55 & \cellcolor{blue!50}11.52 & 339.35 & TO \\
\hline

\multirow{5}{*}{lymph}
 & 2 & \cellcolor{blue!50}84.53 & 82.25 & 7 & 3 & \cellcolor{blue!50}100 & 84.39 & \cellcolor{blue!50}83.89 & \cellcolor{blue!50}4 & \cellcolor{blue!50}0.29 & \cellcolor{blue!50}0.01 & \cellcolor{blue!50}100 & 84.46 & 83.23 & \cellcolor{blue!50}4 & 3.50 & 2.84\\
 & 3 & \cellcolor{blue!50}90.07 & 78.34 & 14.92 & 6.52 & \cellcolor{blue!50}100 & 86.62 & \cellcolor{blue!50}81.59 & 6 & \cellcolor{blue!50}0.88 & \cellcolor{blue!50}0.29 & 32 & 86.76 & 78.35 & \cellcolor{blue!50}5.92 & 5.55 & 829.42\\
 & 4 & \cellcolor{blue!50}95.1 & 80.91 & 25.80 & 11.76 & \cellcolor{blue!50}100 & 90.40 & \cellcolor{blue!50}84.97 & \cellcolor{blue!50}8.36 & \cellcolor{blue!50}2.16 & \cellcolor{blue!50}17.38 & 0 & 90.54 & 82.40 & 8.72 & 7.86 & TO \\
 & 5 & \cellcolor{blue!50}98.34 & 81.58 & 33.40 & 14.96 & \cellcolor{blue!50}88 & 92.97 & \cellcolor{blue!50}83.62 & \cellcolor{blue!50}12.88 & \cellcolor{blue!50}4.21 & \cellcolor{blue!50}277.57 & 0 & 93.51 & 83.60 & 13.52 & 10.94 & TO \\
 & 6 & \cellcolor{blue!50}99.76 & 81.02 & 36.76 & 15.76 & \cellcolor{blue!50}40 & 95.37 & \cellcolor{blue!50}85.33 & \cellcolor{blue!50}17.64 & \cellcolor{blue!50}7.78 & \cellcolor{blue!50}626.2 & 0 & 95.88 & 84.82 & \cellcolor{blue!50}17.64 & 15.74 & TO \\
\hline

\multirow{5}{*}{mushroom}
 & 2 & 92.71 & 92.71 & 7 & 3 & \cellcolor{blue!50}100 & 91.83 & 91.83 & \cellcolor{blue!50}4 & \cellcolor{blue!50}15.61 & \cellcolor{blue!50}6.82 & \cellcolor{blue!50}100 & \cellcolor{blue!50}95.13 & \cellcolor{blue!50}95.13 & \cellcolor{blue!50}4 & 299.19 & 425.74\\
 & 3 & 96.56 & 96.54 & 11 & 5 & \cellcolor{blue!50}100 & 95.37 & 95.37 & \cellcolor{blue!50}5.8 & \cellcolor{blue!50}40.33 & \cellcolor{blue!50}23.11 & 0 & \cellcolor{blue!50}97.74 & \cellcolor{blue!50}97.77 & 6.80 & 458.10 & TO \\
 & 4 & \cellcolor{blue!50}99.95 & \cellcolor{blue!50}99.94 & 16.92 & 7.96 & \cellcolor{blue!50}100 & 98.52 & 98.52 & \cellcolor{blue!50}8.24 & \cellcolor{blue!50}93.45 & \cellcolor{blue!50}89.39 & 0 & 98.78 & 98.74 & 9 & 635.09 & TO \\
 & 5 & \cellcolor{blue!50}99.96 & \cellcolor{blue!50}99.94 & 18.92 & 8.96 & \cellcolor{blue!50}96 & 99.41 & 99.41 & \cellcolor{blue!50}11.12 & \cellcolor{blue!50}183.13 & \cellcolor{blue!50}458.08 & 0 & 98.63 & 98.57 & 11.32 & 853.68 & TO \\
 & 6 & \cellcolor{blue!50}99.97 & \cellcolor{blue!50}99.96 & 20.92 & 9.76 & \cellcolor{blue!50}4 & 99.70 & 99.69 & \cellcolor{blue!50}13.32 & \cellcolor{blue!50}367.46 & \cellcolor{blue!50}821.83 & \cellcolor{blue!50}4 & 97.28 & 97.10 & 14.60 & 1165.88 & 855.34\\
\hline

\multirow{5}{*}{tumor}
 & 2 & 82.77 & 80.65 & 7 & 3 & \cellcolor{blue!50}100 & 82.75 & 80.83 & \cellcolor{blue!50}4 & \cellcolor{blue!50}0.66 & \cellcolor{blue!50}0.05 & \cellcolor{blue!50}100 & \cellcolor{blue!50}82.8 & \cellcolor{blue!50}81.6 & \cellcolor{blue!50}4 & 3.72 & 5.46\\
 & 3 & \cellcolor{blue!50}84.94 & 79.76 & 15 & 6.72 & \cellcolor{blue!50}100 & 83.71 & 80.06 & 5.56 & \cellcolor{blue!50}1.99 & \cellcolor{blue!50}3.11 & 0 & 83.84 & \cellcolor{blue!50}80.43 & \cellcolor{blue!50}5.3 & 6.02 & TO \\
 & 4 & \cellcolor{blue!50}87.01 & 81.30 & 28.68 & 11.96 & \cellcolor{blue!50}76 & 85.46 & 82.38 & \cellcolor{blue!50}8.48 & \cellcolor{blue!50}4.82 & \cellcolor{blue!50}510.07 & 0 & 85.52 & \cellcolor{blue!50}82.49 & 8.64 & 9.04 & TO \\
 & 5 & \cellcolor{blue!50}89.42 & 81.31 & 48.20 & 17.52 & 0 & 87.47 & 84.69 & \cellcolor{blue!50}13.12 & \cellcolor{blue!50}10.06 & \cellcolor{blue!50}TO  & 0 & 87.51 & \cellcolor{blue!50}85.83 & 13.32 & 13.79 & \cellcolor{blue!50}TO \\
 & 6 & \cellcolor{blue!50}91.34 & 80.24 & 68.60 & 21.04 & 0 & 88.60 & \cellcolor{blue!50}81.54 & 20.52 & \cellcolor{blue!50}19.12 & \cellcolor{blue!50}TO  & 0 & 88.57 & 81.12 & \cellcolor{blue!50}19.84 & 22.44 & \cellcolor{blue!50}TO \\
\hline

\multirow{5}{*}{soybean}
 & 2 & 89.13 & 88.54 & 6.92 & 2.76 & \cellcolor{blue!50}100 & 89.30 & 88.76 & \cellcolor{blue!50}4 & \cellcolor{blue!50}1.17 & \cellcolor{blue!50}0.05 & \cellcolor{blue!50}100 & \cellcolor{blue!50}90.48 & \cellcolor{blue!50}90.48 & \cellcolor{blue!50}4 & 10.79 & 9.51\\
 & 3 & \cellcolor{blue!50}92.11 & \cellcolor{blue!50}90.95 & 13.08 & 5.60 & \cellcolor{blue!50}100 & 90.71 & 90.79 & \cellcolor{blue!50}5 & \cellcolor{blue!50}3.34 & \cellcolor{blue!50}1.22 & 68 & 91.39 & 90.41 & 6.52 & 16.99 & 706.39\\
 & 4 & \cellcolor{blue!50}94.36 & 93.05 & 21 & 9.40 & \cellcolor{blue!50}100 & 92.44 & 91.75 & \cellcolor{blue!50}7.92 & \cellcolor{blue!50}7.9 & \cellcolor{blue!50}96.96 & 0 & 93.24 & \cellcolor{blue!50}93.21 & 9.04 & 24.54 & TO \\
 & 5 & \cellcolor{blue!50}96.29 & \cellcolor{blue!50}94.13 & 31.88 & 13.40 & \cellcolor{blue!50}28 & 93.90 & 93.37 & \cellcolor{blue!50}11.08 & \cellcolor{blue!50}16.58 & \cellcolor{blue!50}804.03 & 0 & 94.31 & 92.95 & 11.92 & 35.34 & TO \\
 & 6 & \cellcolor{blue!50}97.58 & 94.89 & 43.48 & 16.88 & 0 & 95.65 & 95.24 & \cellcolor{blue!50}14.56 & \cellcolor{blue!50}33.04 & \cellcolor{blue!50}TO  & 0 & 96.07 & \cellcolor{blue!50}95.52 & 14.88 & 53.41 & \cellcolor{blue!50}TO \\
\hline

\multirow{5}{*}{splice-1}
 & 2 & \cellcolor{blue!50}84.04 & \cellcolor{blue!50}84.04 & 7 & 3 & \cellcolor{blue!50}100 & \cellcolor{blue!50}84.04 & \cellcolor{blue!50}84.04 & \cellcolor{blue!50}4 & \cellcolor{blue!50}6.13 & \cellcolor{blue!50}4.06 & 0 & \cellcolor{blue!50}84.04 & \cellcolor{blue!50}84.04 & \cellcolor{blue!50}4 & 296.61 & TO \\
 & 3 & \cellcolor{blue!50}91.34 & \cellcolor{blue!50}90.85 & 15 & 6.20 & \cellcolor{blue!50}100 & 89.31 & 89.31 & \cellcolor{blue!50}5.32 & \cellcolor{blue!50}17.7 & \cellcolor{blue!50}67.13 & 0 & 87.25 & 86.94 & 5.44 & 449.04 & TO \\
 & 4 & \cellcolor{blue!50}95.44 & \cellcolor{blue!50}95.34 & 28.52 & 10.60 & \cellcolor{blue!50}4 & 92.92 & 92.92 & 7.36 & \cellcolor{blue!50}42.15 & \cellcolor{blue!50}841.79 & 0 & 88.30 & 88.04 & \cellcolor{blue!50}7.24 & 608.30 & TO \\
 & 5 & \cellcolor{blue!50}96.29 & \cellcolor{blue!50}95.52 & 46.92 & 17.20 & 0 & 94.48 & 94.29 & \cellcolor{blue!50}9.68 & \cellcolor{blue!50}93.08 & \cellcolor{blue!50}TO  & 0 & 71.99 & 70.53 & 10.28 & 783.90 & \cellcolor{blue!50}TO \\
 & 6 & \cellcolor{blue!50}97.45 & \cellcolor{blue!50}95.44 & 76.44 & 29.64 & 0 & 79.13 & 78.22 & \cellcolor{blue!50}15.12 & \cellcolor{blue!50}205.48 & \cellcolor{blue!50}TO  & 0 & 62.92 & 61.89 & 16.28 & 996.27 & \cellcolor{blue!50}TO \\
\hline

\multirow{5}{*}{tic-tac-toe}
 & 2 & 70.80 & \cellcolor{blue!50}68.58 & 7 & 2.96 & \cellcolor{blue!50}100 & 70.80 & \cellcolor{blue!50}68.58 & \cellcolor{blue!50}3.76 & \cellcolor{blue!50}1.84 & \cellcolor{blue!50}0.65 & \cellcolor{blue!50}100 & \cellcolor{blue!50}71.05 & 68.35 & 4 & 9.25 & 48.01\\
 & 3 & \cellcolor{blue!50}75.65 & \cellcolor{blue!50}73.21 & 15 & 6.76 & \cellcolor{blue!50}100 & 74.26 & 72.90 & \cellcolor{blue!50}6.04 & \cellcolor{blue!50}5.6 & \cellcolor{blue!50}92.88 & 0 & 74.91 & 72.36 & 6.16 & 15.01 & TO \\
 & 4 & \cellcolor{blue!50}84.14 & \cellcolor{blue!50}81.8 & 27 & 11.04 & 0 & 76.89 & 73.86 & 9.04 & \cellcolor{blue!50}12.99 & \cellcolor{blue!50}TO  & 0 & 76.87 & 74.22 & \cellcolor{blue!50}8.84 & 22.88 & \cellcolor{blue!50}TO \\
 & 5 & \cellcolor{blue!50}91.7 & \cellcolor{blue!50}89.73 & 43.72 & 16 & 0 & 79.79 & 77.02 & 14.24 & \cellcolor{blue!50}27.15 & \cellcolor{blue!50}TO  & 0 & 81.86 & 80.31 & \cellcolor{blue!50}13.88 & 35.67 & \cellcolor{blue!50}TO \\
 & 6 & \cellcolor{blue!50}95.46 & \cellcolor{blue!50}92.36 & 66.68 & 19.28 & 0 & 82.97 & 79.54 & \cellcolor{blue!50}24.04 & \cellcolor{blue!50}52.34 & \cellcolor{blue!50}TO  & 0 & 84.82 & 80.08 & 24.16 & 59.52 & \cellcolor{blue!50}TO \\
\hline

\multirow{5}{*}{vote}
 & 2 & 95.63 & 94.90 & 7 & 3 & \cellcolor{blue!50}100 & 95.66 & \cellcolor{blue!50}95.31 & \cellcolor{blue!50}3.12 & \cellcolor{blue!50}0.85 & \cellcolor{blue!50}0.02 & \cellcolor{blue!50}100 & \cellcolor{blue!50}95.68 & 95.22 & 3.76 & 7.20 & 0.70\\
 & 3 & \cellcolor{blue!50}96.84 & 94.94 & 14.44 & 6.56 & \cellcolor{blue!50}100 & 96.53 & \cellcolor{blue!50}95.03 & \cellcolor{blue!50}5.56 & \cellcolor{blue!50}2.52 & \cellcolor{blue!50}0.35 & \cellcolor{blue!50}100 & 96.69 & 94.57 & \cellcolor{blue!50}5.56 & 11.38 & 94.11\\
 & 4 & \cellcolor{blue!50}97.86 & 94.80 & 24.60 & 10.72 & \cellcolor{blue!50}100 & 97.16 & \cellcolor{blue!50}94.99 & \cellcolor{blue!50}8 & \cellcolor{blue!50}5.86 & \cellcolor{blue!50}13.07 & 8 & 97.40 & 94.39 & 8.16 & 16.49 & TO \\
 & 5 & \cellcolor{blue!50}98.87 & 94.48 & 34.20 & 14.16 & \cellcolor{blue!50}92 & 97.91 & \cellcolor{blue!50}94.62 & \cellcolor{blue!50}12.24 & \cellcolor{blue!50}11.77 & \cellcolor{blue!50}280.12 & 0 & 98.21 & 94.57 & 12.40 & 23.83 & TO \\
 & 6 & \cellcolor{blue!50}99.49 & \cellcolor{blue!50}94.16 & 40.04 & 16.60 & \cellcolor{blue!50}60 & 98.69 & 93.84 & \cellcolor{blue!50}18.08 & \cellcolor{blue!50}22.77 & \cellcolor{blue!50}507.15 & 4 & 98.93 & 93.98 & 18.44 & 36.20 & TO \\
\hline
    
    \end{tabular}
    \caption{\label{tab:annex2_cart_hmaxsatbdd_maxsatbdd}
    \footnotesize{Details of evaluation between \cart{}, heuristic \MaxSat-\Bdd{}, and the original \MaxSat-\Bdd{}.}
    }
\end{table*}





\section{Conclusion}

We propose exact and heuristic methods for optimizing binary decision diagrams (\Bdd{}s) based on the (Maximum) Boolean Satisfiability framework.
Our large experimental studies show clear benefits of the proposed approach in terms of prediction quality and interpretability (compact topology) compared to the existing (heuristic) approaches.

In the future, it would be interesting to extend the proposed approach for multi-valued classification. 
Moreover, a deeper investigation of \Bdd{}s with other interpretable models (such as decision rules and decision sets) is needed for the sake of explainable AI. 



\bibliography{aaai22.bib}

\clearpage
\newpage
\appendix
\section*{Appendix: Algorithm for Producing A \Bdd{} Based on the Beads of a Truth table}

Recall that Proposition \ref{prop:bead_bdd_node} states that there is a one-to-one correspondence between the vertices of a BDD and the beads of the correspondent Boolean function. 
Consider a string $\Truthtable{}$ of length $2^{\depth{}}$ associated to a sequence of variables $[\avarind{1}, \avarind{2}, \dots, \avarind{\depth{}}]$.
This appendix describes the algorithm we used to construct a \Bdd{} of a maximum depth $\depth{}$ using the beads of $\Truthtable{}$.

As the \Bdd{} built is \textit{ordered} and \textit{reduced}, it contains no isomorphic subtrees.
This algorithm creates nodes level by level in a \textit{breadth-first} way. 

The \Bdd{} constructed is defined by \textit{a list of nodes} and \textit{a list of edges}.
Each node is a pair {(\textit{node\_id}, \textit{variable})}. 
The value of \textit{node\_id} is a unique integer called the id of the node which is non-negative for non terminal nodes. There are two sink nodes: 
{($-1, 1$)} associated to the value $1$ (i.e., positive class in the context of binary classification); and  {($-2, 0$)} associated to for the value $0$ (i.e., negative class in binary classification).
Each edge is a tuple {$(p, c , direction)$}, where $p$ is the id of the parent node, $c$ is the if of the child, and $direction \in \{left,right\}$) indicates if $c$ is the left or right child of $p$.


Before we present the algorithm, we show some predefined functions used in the algorithm in Table \ref{tab:lambda_functions_algo}.

\begin{table}[htb]
    \centering
    \small
\setlength{\tabcolsep}{1pt}
    \begin{tabular}{|l||l|}
        \hline
        \footnotesize{\textbf{Function(Input): Output}} & \footnotesize{\textbf{Description}} \\\hline
        $FirstHalf(string~s)$: string & Returns the first half of $s$ \\\hline
        $SecondHalf(string~s)$: string & Returns the second half of $s$ \\\hline
        $IsBead(string~s)$: Boolean & Returns \textit{True} iff $s$ is a bead \\\hline
        $LeadToZero(string~s)$: Boolean & Returns \textit{True} iff $s$ contains only $0$ \\\hline
        $LeadToOne(string~s)$: Boolean & Returns \textit{True} iff $s$ contains only $1$ 
        \\\hline
    \end{tabular}
    \caption{Predefined functions in the algorithm}
    \label{tab:lambda_functions_algo}
\end{table}

The detailed algorithm is described in Algorithm \ref{alg:beads_to_bdd}.
We use a FIFO queue $q$ in the algorithm. 
Each item in the queue follows the format \textit{(str, parent\_id, current\_level, direction)}.
The first item pushed in the queue is a special case associated to the root denoted by $(\Truthtable{}, 0, 1, \emptyset))$ since the root has not parent. 


At each iteration of the main loop, the algorithm pops an element $(s, parent\_id, level, direction)$ from the queue at 
Line~\ref{line:pop}. 
If $s$ is a bead, the algorithm creates a new node at Line~\ref{line:newnode} associated with the level $level$ if $s$ is seen for the first time.  
The set of edges is updated in Line~\ref{line:newedge} accordingly.
The left and right children of $s$ are added in the queue in Lines~\ref{line:left} and ~\ref{line:right}. 

When the current string $s$ is not a Bead of size $>1$, there might be two cases where $s$ leads immediately to a sink node. Either $s$ contains only $0$s or $s$ contains only $1$s. The two cases are handled (according to the size of $s$) in two parts of the algorithm: from Lines~\ref{line:check} to Line~\ref{line:endsink} 
and from Lines~\ref{line:1} to Line~\ref{line:2}.
The case where $s$ is a bead that is not equivalent to a sink node, the set of edges is updated and only one child of $s$ is added to the queue without creating nodes (since $s$ is a bead).
The algorithm ends when all the elements of the queue are treated.

\corrhao{}

\begin{algorithm}[htb]
    \caption{GenBDD$(\Truthtable{}, \mathcal{X})$, an algorithm to construct a \Bdd{} from a given string $\Truthtable{}$ and variable sequence $\mathcal{X}$}\label{alg:beads_to_bdd}
    \begin{algorithmic}[1]
        \REQUIRE{String $\Truthtable{}$, variable sequence $\mathcal{X} = [\avarind{1}, \dots, \avarind{\depth{}}]$.}
        \ENSURE{BDD(\textit{nodes}, \textit{edges})} 
    
        \STATE $nodes \gets \{\}; edges \gets \{\}; T \gets \{\}$
        \STATE $nodes.append((-1, 1)); nodes.append((-2, 0))$
        \STATE $q \gets Queue()$
        \STATE $q.put((\Truthtable{}, 0, 1, \emptyset)))$ 
        
        \WHILE{$not \  q.empty()$}
            \STATE \label{line:pop} $(s, parent\_id, level, direction) \gets q.pop())$
            
            \IF{$Length(s) > 1 \ and  \ IsBead(s)$}
                \STATE // \textit{When the current string $s$ is a Bead.}
                 \IF{${ s \not\in T} $}
                    \STATE // \textit{$s$ is a new string, i.e., not seen before}
                    \STATE $T.append(s)$
                    \STATE $index \gets T.index(s) + 1$
                    \STATE \label{line:newnode} $nodes.append((index, \avarind{level}))$
                \ENDIF
                \STATE $index \gets T.index(s) + 1$
                \IF{ $parent\_id \geq 1$}
                    \STATE  \label{line:newedge} $edges.append(( parent\_id, index, direction))$
                \ENDIF
                \STATE // \textit{Put the left and right child into the queue.}
                \STATE \label{line:left} $q.put((FirstHalf(s), index, level + 1, left))$
                \STATE \label{line:right} $q.put((FirstHalf(s), index, level + 1, right))$

            \ELSIF{ \label{line:check} $Length(s)>1 \textbf{ and not } \ IsBead(s)$}
                \STATE // \textit{When the current string $s$ is \textbf{not} a Bead.}
                \IF{$LeadToOne(s) \textbf{ or } \ LeadToZero(s)$}
                    \STATE // \textit{$s$ leads to sink nodes}
                    \IF{$LeadToOne(s)$}
                        \STATE  $sink \gets -1$
                    \ELSE
                        \STATE  $sink \gets -2$
                    \ENDIF
                    \IF{$p = 0$}
                        \STATE $edges.append((1, sink, left))$
                        \STATE $edges.append((1, sink, right))$
                    \ELSE
                        \STATE $edges.append((parent\_id, sink, direction))$
                    \ENDIF \label{line:endsink}
                \ELSE
                    \STATE // \textit{Otherwise put the left child into the queue.}
                    \STATE \label{line:normalcase} $q.put((FirstHalf(s), parent\_id, level+1 ~, direction))$
                \ENDIF
            \ELSE
                \STATE // \label{line:1} \textit{The current string is a sink node.}
                \IF{$s=1$}
                        \STATE  $sink \gets -1$
                    \ELSE
                        \STATE  $sink \gets -2$
                    \ENDIF
                \STATE \label{line:2} $edges.append(( parent\_id, sink, direction))$
            \ENDIF
            
        \ENDWHILE
    \end{algorithmic}
\end{algorithm}

    


\end{document}